\documentclass{article}
\usepackage[preprint]{neurips_2025}
\usepackage[utf8]{inputenc}
\usepackage[T1]{fontenc}
\usepackage{hyperref}
\usepackage{url}
\usepackage{booktabs}
\usepackage{amsfonts}
\usepackage{nicefrac}
\usepackage{microtype}
\usepackage{xcolor}
\usepackage{amsmath}
\usepackage{amssymb}
\usepackage{amsthm}
\usepackage{algorithm}
\usepackage{algorithmic}
\usepackage{graphicx}
\usepackage{subcaption}
\usepackage{multirow}
\usepackage{enumitem}

\usepackage[most]{tcolorbox}
\definecolor{findingbg}{RGB}{255, 248, 240}
\definecolor{findingframe}{RGB}{230, 126, 34}
\definecolor{takeawaybg}{RGB}{232, 245, 233}
\definecolor{takeawayframe}{RGB}{46, 125, 50}
\definecolor{insightbg}{RGB}{232, 240, 254}
\definecolor{insightframe}{RGB}{30, 136, 229}

\newtcolorbox{findingbox}[1][]{%
  enhanced,
  colback=findingbg,
  colframe=findingframe,
  boxrule=1pt,
  arc=3pt,
  left=8pt, right=8pt, top=6pt, bottom=6pt,
  fonttitle=\bfseries,
  title={#1},
  coltitle=findingframe,
  attach boxed title to top left={xshift=10pt, yshift=-\tcboxedtitleheight/2},
  boxed title style={colback=white, colframe=findingframe, boxrule=0.5pt, arc=2pt}
}

\newtcolorbox{takeawaybox}[1][]{%
  enhanced,
  colback=takeawaybg,
  colframe=takeawayframe,
  boxrule=1.5pt,
  arc=4pt,
  left=10pt, right=10pt, top=8pt, bottom=8pt,
  fonttitle=\bfseries\large,
  title={#1},
  coltitle=takeawayframe,
  attach boxed title to top left={xshift=10pt, yshift=-\tcboxedtitleheight/2},
  boxed title style={colback=white, colframe=takeawayframe, boxrule=0.5pt, arc=2pt}
}

\newtcolorbox{insightbox}[1][]{%
  enhanced,
  colback=insightbg,
  colframe=insightframe,
  boxrule=1pt,
  arc=3pt,
  left=8pt, right=8pt, top=6pt, bottom=6pt,
  fonttitle=\bfseries,
  title={#1},
  coltitle=insightframe,
  attach boxed title to top left={xshift=10pt, yshift=-\tcboxedtitleheight/2},
  boxed title style={colback=white, colframe=insightframe, boxrule=0.5pt, arc=2pt}
}

\newtheorem{theorem}{Theorem}
\newtheorem{proposition}{Proposition}

\newtheorem{definition}{Definition}
\newtheorem{remark}{Remark}
\newtheorem{assumption}{Assumption}

\newcommand{\E}{\mathbb{E}}

\newcommand{\KL}{\mathbb{D}_{\mathrm{KL}}}
\newcommand{\piref}{\pi_{\mathrm{ref}}}
\newcommand{\pith}{\pi_\theta}
\newcommand{\piold}{\pi_{\mathrm{old}}}
\newcommand{\Ym}{\mathcal{Y}^{-}}
\newcommand{\Yp}{\mathcal{Y}^{+}}
\newcommand{\cD}{\mathcal{D}}
\newcommand{\Softplus}{\mathrm{Softplus}}
\newcommand{\tpm}[1]{{\scriptsize\,$\pm$\,#1}}

\title{Overconfident Errors Need Stronger Correction:\\Asymmetric Confidence Penalties for Reinforcement Learning}

\author{%
Yuanda Xu\textsuperscript{*\dag} \quad
Hejian Sang\textsuperscript{*} \quad
Zhengze Zhou\textsuperscript{*} \quad
Ran He\textsuperscript{*} \quad
Zhipeng Wang\textsuperscript{\ddag} \\
LinkedIn Corporation, CA, USA \\[6pt]
\small \textsuperscript{*}Co-first authors with equal contribution. \quad
\textsuperscript{\dag}\texttt{ericxu@linkedin.com} \quad
\textsuperscript{\ddag}\texttt{zhipwang@linkedin.com}
}

\begin{document}

\maketitle

\begin{abstract}
Reinforcement Learning with Verifiable Rewards (RLVR) has become the leading paradigm for enhancing reasoning in Large Language Models (LLMs). However, standard RLVR algorithms suffer from a well-documented pathology: while improving Pass@1 through sharpened sampling, they simultaneously narrow the model's reasoning boundary and reduce generation diversity. We identify a root cause that existing methods overlook: the \emph{uniform penalization of errors}. Current approaches---whether data-filtering methods that select prompts by difficulty, or advantage normalization schemes---treat all incorrect rollouts within a group identically. We show that this uniformity allows \emph{overconfident errors}---incorrect reasoning paths that the RL process has spuriously reinforced---to persist and monopolize probability mass, suppressing valid exploratory trajectories.

We propose the \textbf{Asymmetric Confidence-aware Error Penalty (ACE)}, which introduces a per-rollout confidence shift metric $c_i = \log(\pith(y_i|x) / \piref(y_i|x))$ to dynamically modulate negative advantages. Theoretically, we show that ACE's gradient can be decomposed into the gradient of a \emph{selective regularizer} restricted to overconfident errors, plus a well-characterized residual that partially moderates the regularizer's strength (Theorem~\ref{thm:selective_kl}). Experiments fine-tune Qwen2.5-Math-7B \citep{yang2024qwen25}, Qwen3-8B-Base \citep{yang2025qwen3}, and Llama-3.1-8B-Instruct \citep{dubey2024llama3} on the DAPO-Math-17K dataset \citep{yu2025dapo} using GRPO and DAPO with VERL \citep{sheng2024verl}, evaluating on MATH-500 \citep{hendrycks2021math} and AIME 2025. ACE composes with both GRPO and DAPO, consistently improving the full Pass@$k$ spectrum across all three model families and benchmarks.
\end{abstract}

\section{Introduction}
\label{sec:intro}

Reinforcement Learning with Verifiable Rewards (RLVR) \citep{guo2025deepseek, jaech2024openai} has emerged as a primary method for post-training Large Language Models (LLMs) on reasoning tasks. By using binary correctness signals from deterministic verifiers, algorithms such as PPO \citep{schulman2017ppo}, GRPO \citep{shao2024grpo}, and REINFORCE \citep{williams1992reinforce} iteratively refine the model's Chain-of-Thought (CoT) generation \citep{wei2022cot}.

Despite its successes, a growing body of evidence reveals a fundamental tension in RLVR training. While RLVR models excel at Pass@1, they consistently underperform their own base models at Pass@$k$ for large $k$ \citep{chen2021codex, yue2025limit, brown2024passrate}, indicating a \emph{narrowing} of the reasoning boundary rather than an expansion. This phenomenon has been attributed to diversity collapse: the training process concentrates probability mass on a small number of successful reasoning paths, suppressing the broader solution space.

A prominent strategy addresses this: Difficulty-based curriculum learning filters prompts to maximize gradient signal. However, such methods operate at a macro level—selecting which problems to train on—ignoring a critical micro-level distinction: not all errors are equal.

Within incorrect rollouts, we identify distinct regimes: \emph{exploratory errors} (benign stochastic deviations), \emph{self-correcting errors} (paths the model is already abandoning), and \emph{overconfident errors} (spuriously reinforced paths acting as value traps). Standard RLVR penalizes these uniformly. While the global KL penalty $\beta \KL(\pith \| \piref)$ offers some correction, it is symmetric and indiscriminate, suppressing beneficial exploration alongside harmful overconfidence.

\paragraph{Our contribution.} We propose to break this dilemma by introducing \emph{asymmetric} correction at the level of individual rollouts. Our method, \textbf{ACE} (\textbf{A}symmetric \textbf{C}onfidence-aware \textbf{E}rror penalty), dynamically amplifies the penalty for overconfident errors using a per-rollout confidence shift metric, while leaving exploratory and self-correcting errors largely untouched. Concretely, our contributions are:

\begin{enumerate}[leftmargin=2em]
    \item \textbf{A new analytical dimension.} We formalize \emph{error confidence shift} $c_i = \log(\pith(y_i|x) / \piref(y_i|x))$ as a per-rollout diagnostic that is orthogonal to prompt-level difficulty, and show empirically that overconfident errors accumulate during training (\S\ref{sec:prelim}).
    
    \item \textbf{Theoretical foundations.} We show that ACE's gradient admits a decomposition into a \emph{selective regularizer} targeting the overconfident portion of the policy, plus a residual term that partially moderates the regularizer's correction strength (\S\ref{sec:selective_kl}).
    
    \item \textbf{Empirical validation.} ACE consistently improves the full Pass@$k$ spectrum on MATH-500 and AIME 2025, across three model families (Qwen2.5-Math-7B, Qwen3-8B-Base, and Llama-3.1-8B-Instruct) and two base algorithms (GRPO and DAPO), with particularly strong gains at large $k$, confirming that it preserves and expands the reasoning boundary (\S\ref{sec:experiments}).
\end{enumerate}

\begin{figure}[t]
\centering
\includegraphics[width=0.9\textwidth]{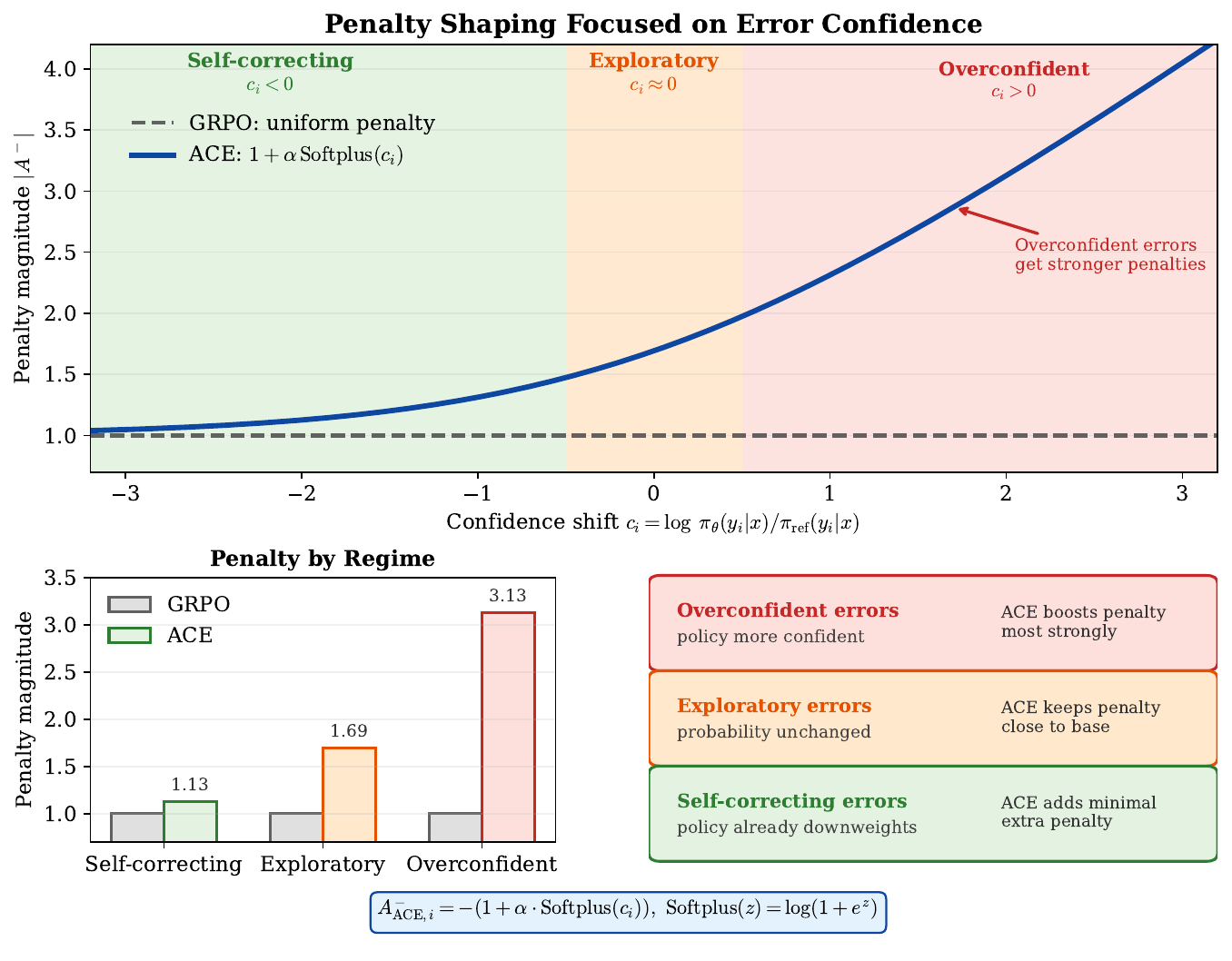}
\caption{\textbf{ACE Method Overview.} \emph{Top:} Incorrect rollouts fall into three regimes based on the confidence shift $c_i = \log(\pith(y_i|x)/\piref(y_i|x))$. \emph{Bottom-left:} Standard GRPO assigns a uniform penalty $|\hat{A}^-|$ to all errors regardless of regime. \emph{Bottom-right:} ACE modulates the penalty via $\text{Softplus}(c_i)$, strongly penalizing overconfident errors while leaving self-correcting errors nearly untouched.}
\label{fig:method_overview}
\end{figure}

Figure~\ref{fig:method_overview} highlights the core idea: ACE reshapes the negative penalty as a smooth, confidence-dependent curve, amplifying penalties for overconfident errors while keeping exploratory and self-correcting errors close to the base level.

\section{Related Work}
\label{sec:related}

\paragraph{Curriculum and advantage shaping.} Curriculum methods \citep{bae2025cures, parashar2025e2h, zhang2025clpo} select prompts by difficulty, operating at the \emph{prompt level}. Advantage shaping methods \citep{li2025rethinking, wen2025implicit} balance correct vs.\ incorrect samples at the \emph{group level}. ACE operates at the \emph{rollout level}, modulating penalties within incorrect samples based on per-rollout confidence shift.

\paragraph{KL regularization in RLHF/RLVR.} KL divergence penalties are standard in RLHF pipelines to prevent reward hacking and mode collapse \citep{ouyang2022instructgpt, stiennon2020summarize}. The typical formulation adds a global term $\beta \KL(\pith \| \piref)$ that symmetrically penalizes all deviations from the reference. DPO \citep{rafailov2023dpo} implicitly constrains the KL divergence through its closed-form reward parameterization. However, all these methods apply KL penalties \emph{uniformly} across correct and incorrect outputs alike, suppressing beneficial exploration alongside harmful overconfidence. ACE introduces an \emph{asymmetric} and \emph{selective} KL-like penalty that targets only overconfident errors while leaving correct outputs and self-correcting errors untouched.

\paragraph{Entropy regularization and clipping strategies.} Entropy bonuses have a long history in RL for encouraging exploration \citep{williams1992reinforce, schulman2017ppo}. In the LLM context, DAPO \citep{yu2025dapo} combats entropy collapse through its Clip-Higher strategy, which decouples the upper and lower clipping thresholds of the importance sampling ratio to give low-probability exploration tokens more room for probability increase. While such clipping-based strategies promote diversity globally, they operate at the \emph{token level} and cannot distinguish between beneficial diversity on correct reasoning paths and harmful persistence of incorrect ones. ACE provides a complementary, more targeted mechanism: rather than modifying the clipping bounds, it modulates the penalty magnitude per \emph{rollout} based on confidence shift, achieving diversity preservation as a \emph{consequence} of selectively suppressing overconfident errors (see \S\ref{sec:exp_entropy}).

\paragraph{Reward shaping.} Potential-based reward shaping \citep{ng1999policy, wiewiora2003principled, devlin2012dynamic} transforms the reward function to accelerate learning while preserving the optimal policy. ACE can be viewed through the reward shaping lens: the confidence-dependent term $\alpha \cdot \Softplus(c_i)$ acts as an auxiliary reward signal derived from the policy--reference divergence. Unlike classical potential-based shaping, ACE's shaping signal is \emph{asymmetric} (applied only to negative advantages) and \emph{adaptive} (it evolves with the policy). Process reward models \citep{lightman2023verify} offer another form of reward enrichment at the step level; ACE is complementary, operating at the trajectory level with zero additional annotation cost.

\paragraph{Diversity loss in RLVR.} Yue et al.\ \citep{yue2025limit} show RLVR narrows reasoning boundaries. Negative sample reinforcement \citep{zhu2025nsr} demonstrates the importance of learning from incorrect rollouts but does not differentiate among error types. We identify overconfident errors as a key mechanism driving diversity collapse and propose confidence-based differential penalization to address this.

\section{Preliminaries}
\label{sec:prelim}

\paragraph{Setting.} We consider a policy $\pith$ parameterized by $\theta$, initialized from a reference model $\piref$. Given a prompt $x \sim \cD$, the model generates $G$ rollouts $\{y_1, \ldots, y_G\}$. Each rollout receives a reward $r_i \in \mathbb{R}$ from a reward function or verifier. We define:
\begin{itemize}[leftmargin=2em]
    \item Empirical mean reward: $\hat{\mu}_x = \frac{1}{G} \sum_{i=1}^G r_i$
    \item Empirical reward standard deviation: $\hat{\sigma}_x = \sqrt{\frac{1}{G} \sum_{i=1}^G (r_i - \hat{\mu}_x)^2}$
    \item Above-average rollout set: $\Yp(x) = \{y_i : r_i > \hat{\mu}_x\}$
    \item Below-average rollout set: $\Ym(x) = \{y_i : r_i \leq \hat{\mu}_x\}$
\end{itemize}

\paragraph{GRPO objective.} In Group Relative Policy Optimization \citep{shao2024grpo}, the advantage for rollout $y_i$ is computed via group normalization:
\begin{equation}
    \hat{A}_i = \frac{r_i - \hat{\mu}_x}{\hat{\sigma}_x + \epsilon}
\end{equation}
where $\epsilon$ is a small constant for numerical stability. The clipped surrogate objective is:
\begin{equation}
\label{eq:grpo}
    \mathcal{L}_{\text{GRPO}}(\theta) = -\E_{x \sim \cD} \left[ \frac{1}{G} \sum_{i=1}^G \min\!\left( \rho_i \hat{A}_i,\; \text{clip}(\rho_i, 1{-}\epsilon_c, 1{+}\epsilon_c) \hat{A}_i \right) \right] + \beta \KL(\pith \| \piref)
\end{equation}
where $\rho_i = \pith(y_i|x) / \piold(y_i|x)$ is the importance sampling ratio and $\epsilon_c$ is the clipping threshold.

\paragraph{Observation: uniform penalty within groups.} For rollouts with identical rewards $r_i = r_j$, the advantages are also identical: $\hat{A}_i = \hat{A}_j$. In the special case of binary rewards, all incorrect rollouts ($r_i = 0$) share the same advantage:
\begin{equation}
    \hat{A}_i^- = \frac{-\hat{\mu}_x}{\hat{\sigma}_x + \epsilon}
\end{equation}
More generally, rollouts with the same reward receive \emph{identical} advantage values regardless of their qualitative differences. The only per-rollout modulation comes from the importance ratio $\rho_i$, which is bounded by clipping and provides limited differentiation.

\paragraph{Motivation: overconfident errors.}
\label{sec:empirical_pathology}
Define the per-rollout \emph{confidence shift} $c_i = \log(\pith(y_i|x) / \piref(y_i|x))$: positive values indicate the policy has become more confident than the reference on rollout $y_i$, while negative values indicate the opposite. Training Qwen2.5-Math-7B with standard GRPO on DAPO-Math-17K \citep{yu2025dapo}, we observe that the distribution of $c_i$ among \emph{incorrect} rollouts develops a heavy right tail as training progresses---a substantial fraction of errors become significantly more probable under the trained policy than under the reference, even though they remain incorrect. This is consistent with analyses of implicit reward distributions in preference optimization \citep{rafailov2024dpo, meng2024simpo}. These overconfident errors consume probability mass that would otherwise support diverse reasoning paths, contributing to the diversity collapse documented by \citet{yue2025limit}. Crucially, the standard global KL penalty $\beta \KL(\pith \| \piref)$ cannot selectively address this: it penalizes \emph{all} deviations from the reference proportionally, suppressing beneficial confidence growth on correct paths alongside harmful overconfidence on incorrect ones. This structural limitation motivates a \emph{targeted} correction mechanism (see \S\ref{sec:exp_dynamics} for detailed quantitative tracking).

\section{The ACE Method}
\label{sec:method}
\subsection{Error Confidence Score}
\label{sec:confidence_score}

\begin{definition}[Error Confidence Score]
\label{def:confidence}
For a prompt $x$ and an incorrect rollout $y_i \in \Ym(x)$, the \textbf{error confidence score} is:
\begin{equation}
    c_i \triangleq \log \frac{\pith(y_i | x)}{\piref(y_i | x)} = \sum_{t=1}^{T_i} \log \frac{\pith(y_i^{(t)} | x, y_i^{(<t)})}{\piref(y_i^{(t)} | x, y_i^{(<t)})}
\end{equation}
where $y_i^{(t)}$ denotes the $t$-th token and $T_i$ is the sequence length.
\end{definition}

The second equality decomposes the sequence-level confidence into a sum of \emph{token-level} log-ratios. This is important for two reasons: (a) it shows that $c_i$ is already computed as a byproduct of standard RLVR training (which requires $\log \pith$ and $\log \piref$ for the KL penalty and importance ratios), incurring \emph{zero additional compute}; and (b) it reveals that $c_i$ aggregates confidence shifts across all reasoning steps, naturally weighting tokens where the policy has diverged most from the reference.

\begin{remark}[Three regimes]
\label{remark:regimes}
The sign of $c_i$ partitions incorrect rollouts into interpretable regimes:
\begin{itemize}[leftmargin=2em]
    \item $c_i > 0$: \textbf{Overconfident errors.} The policy assigns \emph{higher} probability than the reference. These are spurious patterns actively learned during RL.
    \item $c_i \approx 0$: \textbf{Exploratory errors.} Probability approximately unchanged from the reference. Natural stochastic deviations.
    \item $c_i < 0$: \textbf{Self-correcting errors.} The policy has already reduced probability mass relative to the reference.
\end{itemize}
\end{remark}

\subsection{The ACE Advantage}
\label{sec:ace_advantage}

We restructure the negative advantage to depend on the per-rollout confidence score $c_i$.

\begin{definition}[ACE Advantage]
\label{def:ace}
For an incorrect rollout $y_i \in \Ym(x)$, the ACE advantage is:
\begin{equation}
\label{eq:ace_advantage}
    A_{\mathrm{ACE},i}^{-} = \hat{A}_i^{-} \cdot \left(1 + \alpha \cdot \Softplus(c_i)\right)
\end{equation}
where $\hat{A}_i^{-} = (r_i - \hat{\mu}_x)/(\hat{\sigma}_x + \epsilon)$ is the standard GRPO advantage for incorrect rollouts and $\alpha \geq 0$ is a hyperparameter controlling the correction strength. Since $\hat{A}_i^- < 0$ and $(1 + \alpha \cdot \Softplus(c_i)) \geq 1$, ACE strictly amplifies the magnitude of the penalty. For correct rollouts $y_i \in \Yp(x)$, we retain the standard GRPO advantage:
\begin{equation}
    A_{\mathrm{ACE},i}^{+} = \hat{A}_i = \frac{r_i - \hat{\mu}_x}{\hat{\sigma}_x + \epsilon}
\end{equation}
\end{definition}

\paragraph{Design rationale.} The Softplus function $\Softplus(z) = \log(1 + e^z)$ is chosen for three properties:
\begin{enumerate}[leftmargin=2em]
    \item \textbf{Asymptotic behavior.} When $c_i \gg 0$ (overconfident), $\Softplus(c_i) \approx c_i$: penalty scales linearly with the log-confidence ratio. When $c_i \ll 0$ (self-correcting), $\Softplus(c_i) \approx e^{c_i} \to 0$: penalty converges to the base GRPO advantage $\hat{A}_i^-$.
    \item \textbf{Smoothness.} Unlike $\max(0, c_i)$ (which has a non-differentiable kink at 0), Softplus is infinitely differentiable everywhere, ensuring smooth gradient flow.
    \item \textbf{Monotonicity.} $\Softplus$ is strictly increasing, so more confident errors always receive strictly larger penalties, consistent with our theoretical motivation.
\end{enumerate}

\paragraph{Comparison to uniform penalization.} To illustrate the effect of ACE, consider the binary reward case where rollouts receive $r_i \in \{0, 1\}$. Under standard GRPO, all incorrect rollouts ($r_i = 0$) share the same advantage $\hat{A}^- = -\hat{p}_x / (\hat{\sigma}_x + \epsilon)$, where $\hat{p}_x$ is the empirical pass rate and $\hat{\sigma}_x = \sqrt{\hat{p}_x(1 - \hat{p}_x)}$.

\paragraph{Difficulty-adaptive scaling.} Since ACE multiplies the standard GRPO advantage $\hat{A}_i^-$ by $(1 + \alpha \cdot \Softplus(c_i))$, it naturally inherits GRPO's difficulty-dependent scaling: easy prompts (high pass rate) produce larger $|\hat{A}_i^-|$, so errors on easy problems are penalized more heavily. The confidence modulation then provides \emph{additional} per-rollout differentiation within each difficulty level.

\paragraph{Penalty differentiation.} Under ACE:
\begin{equation}
    A_{\mathrm{ACE},i}^- = \hat{A}_i^- \cdot \left(1 + \alpha \log(1 + e^{c_i})\right) \quad \Longrightarrow \quad |A_{\mathrm{ACE},i}^-| \text{ is strictly increasing in } c_i
\end{equation}
Therefore, within the same group, an overconfident error ($c_i = 2$) receives a penalty $|\hat{A}^-| \cdot (1 + \alpha \cdot 2.13)$ while an exploratory error ($c_i = 0$) receives $|\hat{A}^-| \cdot (1 + \alpha \cdot 0.69)$, and a self-correcting error ($c_i = -3$) receives $|\hat{A}^-| \cdot (1 + \alpha \cdot 0.05)$. This provides fine-grained differentiation that is impossible under uniform penalization. The same principle extends to continuous rewards, where ACE differentiates among below-average rollouts based on their confidence scores.

\begin{insightbox}[ACE in One Sentence]
Standard RLVR punishes all wrong answers equally. ACE punishes wrong answers \emph{the model has learned to be confident in} much harder, while leaving natural exploration mistakes alone.
\end{insightbox}

\subsection{ACE-GRPO: Integration and Algorithm}
\label{sec:ace_grpo}

ACE modifies only the advantage computation for negative samples. Substituting the ACE advantage (Definition~\ref{def:ace}) into the GRPO objective (Eq.~\ref{eq:grpo}), the full ACE-GRPO objective is:
\begin{equation}
\label{eq:ace_grpo}
    \mathcal{L}_{\mathrm{ACE}}(\theta) = -\E_{x \sim \cD}\left[\frac{1}{G}\sum_{i=1}^G \left(\mathbb{I}[r_i{=}1] \cdot \mathcal{L}_i^+ + \mathbb{I}[r_i{=}0] \cdot \mathcal{L}_i^-\right)\right] + \beta \KL(\pith \| \piref)
\end{equation}
where:
\begin{align}
    \mathcal{L}_i^+ &= \min\!\left(\rho_i \hat{A}_i^+,\; \text{clip}(\rho_i, 1{-}\epsilon_c, 1{+}\epsilon_c) \hat{A}_i^+\right) \label{eq:psr_loss} \\
    \mathcal{L}_i^- &= \min\!\left(\rho_i A_{\mathrm{ACE},i}^-,\; \text{clip}(\rho_i, 1{-}\epsilon_c, 1{+}\epsilon_c) A_{\mathrm{ACE},i}^-\right) \label{eq:nsr_loss}
\end{align}
The positive advantages $\hat{A}_i^+$ retain the standard GRPO formulation.

\paragraph{Practical considerations.} In practice, we normalize $c_i$ by sequence length ($\bar{c}_i = c_i / T_i$) to ensure comparable penalty magnitudes across rollouts of different lengths. Additional implementation details (sequence-level vs.\ token-level aggregation, clipping choices) and a PyTorch implementation are provided in Appendix~\ref{app:implementation}. The full algorithm is given below.

\begin{algorithm}[t]
\caption{ACE-GRPO: Asymmetric Confidence-aware Error Penalty}
\label{alg:ace}
\begin{algorithmic}[1]
\REQUIRE Policy $\pith$, reference model $\piref$, prompt dataset $\cD$, group size $G$, ACE strength $\alpha$, clipping $\epsilon_c$, KL coefficient $\beta$
\FOR{each training step}
    \STATE Sample prompt batch $\{x_1, \ldots, x_B\} \sim \cD$
    \FOR{each prompt $x$ in batch}
        \STATE Generate $G$ rollouts $\{y_1, \ldots, y_G\} \sim \pith(\cdot | x)$
        \STATE Compute rewards $r_i \in \{0, 1\}$ via verifier
        \STATE Compute standard group advantages $\hat{A}_i$ via GRPO
        \FOR{each incorrect rollout $y_i$ with $r_i = 0$}
            \STATE $c_i \leftarrow \left(\sum_{t=1}^{T_i} \log \pith(y_i^{(t)} | \cdot) - \log \piref(y_i^{(t)} | \cdot)\right) / T_i$ \hfill {\small\color{gray} // Already computed}
            \STATE $A_{\mathrm{ACE},i}^- \leftarrow \hat{A}_i^- \cdot (1 + \alpha \cdot \log(1 + \exp(c_i)))$ \hfill {\small\color{gray} // Amplify uniform advantage}
        \ENDFOR
        \STATE Compute clipped surrogate loss (Eq.~\ref{eq:ace_grpo})
    \ENDFOR
    \STATE Update $\theta$ via gradient descent
\ENDFOR
\end{algorithmic}
\end{algorithm}

\subsection{Relationship to Selective Reverse KL Divergence}
\label{sec:selective_kl}

We now characterize the theoretical relationship between ACE's additional penalty and a selective regularizer that targets overconfident errors. Crucially, the equivalence is not exact: ACE implements a \emph{stop-gradient} (reward-shaping) view of the confidence score, which omits a residual term compared to the full regularizer gradient. We state the exact decomposition below.

\begin{theorem}[Selective Regularization Decomposition]
\label{thm:selective_kl}
Let $\mathcal{L}_{\mathrm{std}}(\theta)$ denote the standard policy gradient objective (Eq.~\ref{eq:grpo}) with uniform negative advantages $\hat{A}^-$, and let $\mathcal{L}_{\mathrm{ACE}}(\theta)$ denote the objective with ACE advantages (Eq.~\ref{eq:ace_advantage}). Define the \textbf{selective regularizer}:
\begin{equation}
\label{eq:selective_reg}
    \mathcal{R}_{\mathrm{sel}}(\theta) = \E_{x \sim \cD}\left[|\hat{A}^-(x)| \sum_{y \in \Ym(x)} \pith(y|x) \cdot \Softplus\!\left(\log \frac{\pith(y|x)}{\piref(y|x)}\right)\right]
\end{equation}
where $|\hat{A}^-(x)|$ is the magnitude of the standard GRPO negative advantage for prompt $x$. Assume rollouts are sampled on-policy from $\pith$. Then, in the infinite-sample limit ($G \to \infty$), the $\alpha$-dependent additional gradient from ACE decomposes \emph{exactly} as:
\begin{equation}
\label{eq:exact_decomposition}
    \Delta \nabla_\theta \;=\; -\alpha \nabla_\theta \mathcal{R}_{\mathrm{sel}}(\theta) \;+\; \alpha\, \E_{x \sim \cD}\!\left[|\hat{A}^-(x)| \sum_{y \in \Ym(x)} \pith(y|x)\, \sigma(c(y))\, \nabla_\theta \log \pith(y|x)\right]
\end{equation}
where $\sigma(c) = 1/(1 + e^{-c})$ is the sigmoid function (i.e., $\Softplus'(c)$). Equivalently, ACE implements the \emph{negative} gradient of $\mathcal{R}_{\mathrm{sel}}$ with the confidence modulation treated as a \emph{fixed reward signal} (stop-gradient on $c_i$), plus the residual term $\mathcal{E}(\theta)$:
\begin{equation}
\label{eq:residual}
    \mathcal{E}(\theta) = \E_{x \sim \cD}\!\left[|\hat{A}^-(x)| \sum_{y \in \Ym(x)} \pith(y|x)\, \sigma(c(y))\, \nabla_\theta \log \pith(y|x)\right]
\end{equation}

Moreover, for overconfident errors where $c(y) \gg 0$, $\Softplus(c) \approx c$, and the dominant term in $\mathcal{R}_{\mathrm{sel}}$ takes the form of a \emph{difficulty-weighted} reverse KL divergence restricted to overconfident incorrect trajectories:
\begin{equation}
\label{eq:approx_kl}
    \mathcal{R}_{\mathrm{sel}}(\theta) \;\approx\; \E_{x \sim \cD}\left[|\hat{A}^-(x)| \sum_{\substack{y \in \Ym(x) \\ c(y) > 0}} \pith(y|x) \cdot \log \frac{\pith(y|x)}{\piref(y|x)}\right]
\end{equation}
\end{theorem}

The proof is provided in Appendix~\ref{app:selective_kl_proof}. Intuitively, ACE's stop-gradient treatment of $\Softplus(c_i)$ captures the dominant selective-regularization component (Term~I: confidence-weighted probability suppression), while the residual $\mathcal{E}(\theta)$ corresponds to the through-$c_i$ gradient (Term~II) that the full regularizer would additionally apply. By omitting Term~II, ACE implements a \emph{tempered} version of $\mathcal{R}_{\mathrm{sel}}$---less aggressive than the full regularizer, but more targeted than standard GRPO. The per-prompt factor $|\hat{A}^-(x)|$ ensures that the selective regularizer inherits the difficulty-adaptive scaling of GRPO. In contrast to the global KL term $\beta \KL(\pith \| \piref)$ which indiscriminately pulls back \emph{all} deviations, $\mathcal{R}_{\mathrm{sel}}$ is (i) restricted to incorrect outputs ($y \in \Ym$), (ii) activated primarily by overconfidence ($c_i > 0$) due to Softplus saturation, (iii) independently tunable via $\alpha$, and (iv) difficulty-adaptive via the $|\hat{A}^-(x)|$ factor.

\subsection{Gradient Quality Analysis}
\label{sec:variance_analysis}

A natural question is whether ACE's confidence-dependent reweighting improves or degrades gradient quality. We analyze this in detail in Appendix~\ref{app:gradient_quality} and summarize the key results here.

First, ACE necessarily increases both the total gradient second moment and the directional variance---unavoidable consequences of additive reweighting where $(1 + \alpha \phi_i) > 1$ for all $\phi_i > 0$ (Proposition~\ref{prop:second_moment}). However, this does \emph{not} prevent quality improvement. We define the gradient quality ratio as $Q_d = \mu_d^2 / \sigma_d^2$, measuring the ratio of squared directional signal to directional variance. Under realistic conditions---specifically, when overconfident errors carry gradients aligned with the optimization direction ($\mathrm{Cov}(\phi_i, u_i) > 0$) and the baseline gradient is noisy ($Q_d^{\mathrm{std}} < 1$)---we prove that ACE strictly improves gradient quality: $Q_d^{\mathrm{ACE}} > Q_d^{\mathrm{std}}$ (Theorem~\ref{thm:variance_bound}). The key mechanism is that ACE's selective amplification concentrates extra weight on the most informative gradients, causing the signal to grow faster than the noise along the optimization-relevant direction.

\begin{insightbox}[Theoretical Insight: Gradient Efficiency]
ACE converts \emph{harmful variance into exploitable signal}. By concentrating penalty weight on overconfident errors---which have gradients more aligned with the optimization direction---ACE achieves higher signal-to-noise ratio despite increasing total gradient variance. This is the statistical foundation for ACE's improved learning efficiency.
\end{insightbox}

\section{Experiments: ACE Expands the Reasoning Boundary}
\label{sec:experiments}

\subsection{Experimental Setup}

\paragraph{Models.} We fine-tune Qwen2.5-Math-7B \citep{yang2024qwen25}, Qwen3-8B-Base \citep{yang2025qwen3}, and Llama-3.1-8B-Instruct \citep{dubey2024llama3} using GRPO implemented with VERL \citep{sheng2024verl}. Note that Qwen3-8B-Base is evaluated \emph{without} enabling the extended thinking mode (i.e., reasoning mode disabled). Llama-3.1-8B-Instruct is included in the main results (Tables~\ref{tab:math500_results} and \ref{tab:math_aime25_results}) to test cross-family generalization beyond the Qwen model family. For the detailed diagnostic experiments (\S\ref{sec:exp_dynamics}--\S\ref{sec:exp_entropy}), ablations (\S\ref{sec:ablation_modulation}), and hyperparameter sensitivity (Appendix~\ref{app:alpha_sensitivity}), we focus on the two Qwen models because: (i) they serve as the primary experimental subjects and already span two distinct pretraining recipes (math-specialized vs.\ general-purpose base model), providing sufficient diversity to validate the generality of our findings; and (ii) Llama-3.1-8B-Instruct operates in a substantially lower accuracy regime (e.g., near-floor on AIME 2025), which makes fine-grained diagnostics such as overconfident error distributions and entropy dynamics less statistically informative.

\paragraph{Training data.} We use the DAPO-Math-17K dataset \citep{yu2025dapo} as the training prompts. For the GRPO and ACE-GRPO baselines we use standard GRPO (symmetric clipping, with KL penalty); for DAPO and ACE-DAPO we use the full DAPO algorithm \citep{yu2025dapo} with Clip-Higher, dynamic sampling, and token-level loss.

\paragraph{Evaluation.} We evaluate on MATH-500 \citep{hendrycks2021math} and AIME 2025 using a rule-based math verifier for correctness verification.

\paragraph{Metrics.} We report Pass@$k$ for $k \in \{1, 2, 4, 8, 16, 32\}$ using temperature 0.7 and top-$p$ = 0.95. Pass@$k$ measures the probability that at least one of $k$ samples is correct. We use the unbiased estimator from \citet{chen2021codex}:
\begin{equation}
    \text{Pass@}k = \E_{x \sim \cD}\left[1 - \frac{\binom{n-c}{k}}{\binom{n}{k}}\right]
\end{equation}
where $n$ is the total samples and $c$ is the number correct. Pass@1 reflects exploitation; large-$k$ reflects exploration and reasoning boundary.

\paragraph{Baselines.}
\begin{itemize}[leftmargin=2em]
    \item \textbf{Base model}: Unmodified pretrained model (upper bound for large-$k$ diversity).
    \item \textbf{GRPO}: Standard Group Relative Policy Optimization \citep{shao2024grpo}.
    \item \textbf{DAPO}: The full DAPO algorithm \citep{yu2025dapo}, which uses asymmetric clipping (Clip-Higher), dynamic sampling, and token-level loss, trained on the same DAPO-Math-17K dataset.
    \item \textbf{ACE-GRPO}: Our method (ACE applied to GRPO).
    \item \textbf{ACE-DAPO}: Our method applied on top of DAPO, demonstrating composability with orthogonal diversity-preserving strategies.
\end{itemize}

\paragraph{Hyperparameters.} For ACE, we set $\alpha = 1.0$ as the default. We use normalized confidence scores $\bar{c}_i = c_i / T_i$. Full training hyperparameters are provided in Appendix~\ref{app:hyperparameters}.

\paragraph{Fair comparison.} Within each model, we keep the training recipe and budget matched across methods; see Appendix~\ref{app:hyperparameters}.
\subsection{Main Results: Full Pass@k Spectrum}
\label{sec:main_results}

Table~\ref{tab:math500_results} reports Pass@$k$ on MATH-500 and Table~\ref{tab:math_aime25_results} reports results on AIME 2025.

\begin{table}[t]
\centering
\caption{Pass@$k$ (\%) on MATH-500. We report mean $\pm$ 95\% confidence interval over 5 independent training runs. \textbf{Bold} = best within each model group; \underline{underline} = second best.}
\label{tab:math500_results}
\resizebox{\textwidth}{!}{%
\begin{tabular}{lccccccc}
\toprule
\textbf{Model} & \textbf{@1} & \textbf{@2} & \textbf{@4} & \textbf{@8} & \textbf{@16} & \textbf{@32} \\
\midrule
Qwen2.5-Math-7B & 63.0 & 76.3 & 83.2 & 88.1 & 91.2 & 93.5 \\
Qwen2.5-Math-7B + GRPO & 73.4\tpm{0.8} & 79.5\tpm{0.7} & 83.2\tpm{0.7} & 86.2\tpm{0.5} & 89.7\tpm{0.5} & 91.3\tpm{0.5} \\
Qwen2.5-Math-7B + DAPO & 74.5\tpm{1.0} & 80.8\tpm{0.9} & 84.8\tpm{0.8} & 89.5\tpm{0.7} & 93.0\tpm{0.7} & 94.6\tpm{0.6} \\
Qwen2.5-Math-7B + ACE-GRPO & 74.2\tpm{0.7} & 80.9\tpm{0.7} & 84.5\tpm{0.6} & 88.9\tpm{0.5} & 92.6\tpm{0.5} & 94.3\tpm{0.4} \\
Qwen2.5-Math-7B + ACE-DAPO & \textbf{75.1}\tpm{0.8} & \textbf{82.4}\tpm{0.8} & \textbf{86.2}\tpm{0.6} & \textbf{91.2}\tpm{0.5} & \textbf{94.7}\tpm{0.5} & \textbf{96.1}\tpm{0.5} \\
\midrule
Qwen3-8B-Base & 60.2 & 72.5 & 78.8 & 83.9 & 87.2 & 90.6 \\
Qwen3-8B-Base + GRPO & 69.4\tpm{0.9} & 75.5\tpm{0.9} & 79.3\tpm{0.9} & 82.5\tpm{0.9} & 86.2\tpm{0.8} & 88.6\tpm{0.7} \\
Qwen3-8B-Base + DAPO & 70.8\tpm{1.0} & 76.8\tpm{0.9} & 81.1\tpm{0.9} & 84.3\tpm{0.8} & 88.1\tpm{0.8} & 90.4\tpm{0.8} \\
Qwen3-8B-Base + ACE-GRPO & 70.1\tpm{0.7} & 76.5\tpm{0.7} & 81.1\tpm{0.6} & 84.5\tpm{0.6} & 88.5\tpm{0.6} & 91.1\tpm{0.5} \\
Qwen3-8B-Base + ACE-DAPO & \textbf{71.2}\tpm{0.9} & \textbf{77.5}\tpm{0.9} & \textbf{82.3}\tpm{0.8} & \textbf{85.4}\tpm{0.8} & \textbf{89.4}\tpm{0.7} & \textbf{91.6}\tpm{0.7} \\
\midrule

Llama-3.1-8B-Instruct & 48.1 & 59.9 & 67.8 & \textbf{74.8} & \textbf{80.5} & \textbf{84.8} \\
Llama-3.1-8B-Instruct + GRPO & 52.9\tpm{1.1} & 60.5\tpm{1.0} & 67.3\tpm{0.9} & 71.8\tpm{0.9} & 75.5\tpm{0.9} & 79.3\tpm{0.8} \\
Llama-3.1-8B-Instruct + DAPO & \underline{54.3}\tpm{1.0} & 61.8\tpm{1.0} & 68.9\tpm{1.0} & 72.9\tpm{1.0} & 76.8\tpm{0.9} & 80.4\tpm{0.9} \\
Llama-3.1-8B-Instruct + ACE-GRPO & 54.1\tpm{1.1} & \underline{62.2}\tpm{1.1} & \underline{69.1}\tpm{1.0} & 73.5\tpm{1.0} & 76.8\tpm{0.9} & 81.5\tpm{0.9} \\
Llama-3.1-8B-Instruct + ACE-DAPO & \textbf{55.4}\tpm{1.1} & \textbf{62.8}\tpm{1.1} & \textbf{70.2}\tpm{1.0} & \underline{74.1}\tpm{1.0} & \underline{77.9}\tpm{0.9} & \underline{82.1}\tpm{0.9} \\
\bottomrule
\end{tabular}%
}
\end{table}

\begin{table}[t]
\centering
\caption{Pass@$k$ (\%) on AIME 2025. AIME 2025 contains 30 problems; we report point estimates as confidence intervals are dominated by test-set size rather than training variance. \textbf{Bold} = best within each model group; \underline{underline} = second best.}
\label{tab:math_aime25_results}
\begin{tabular}{lccccccc}
\toprule
\textbf{Model} & \textbf{1} & \textbf{2} & \textbf{4} & \textbf{8} & \textbf{16} & \textbf{32} \\
\midrule
Qwen2.5-Math-7B & 6.3 & 9.9 & 13.8 & 17.5 & 21.9 & 26.7 \\
Qwen2.5-Math-7B + GRPO & 10.5 & 14.9 & 19.7 & 23.9 & 28.6 & 33.7 \\
Qwen2.5-Math-7B + DAPO & 11.5 & 16.7 & 22.5 & 27.5 & 31.8 & 37.1 \\
Qwen2.5-Math-7B + ACE-GRPO & 11.2 & 16.0 & 21.2 & 26.1 & 30.6 & 36.4 \\
Qwen2.5-Math-7B + ACE-DAPO & \textbf{11.7} & \textbf{17.4} & \textbf{23.8} & \textbf{28.5} & \textbf{33.1} & \textbf{38.6} \\
\midrule
Qwen3-8B-Base & 5.1 & 9.2 & 11.6 & 14.2 & 17.0 & 19.6  \\
Qwen3-8B-Base + GRPO & 9.7 & 13.9 & 17.4 & 22.5 & 25.7 & 29.8 \\
Qwen3-8B-Base + DAPO & 11.1 & 15.7 & 19.9 & 25.2 & 28.5 & 33.1 \\ 
Qwen3-8B-Base + ACE-GRPO & 10.5 & 15.5 & 19.6 & 24.7 & 27.9 & 32.4 \\
Qwen3-8B-Base + ACE-DAPO & \textbf{11.2} & \textbf{16.9} & \textbf{21.2} & \textbf{26.3} & \textbf{29.9} & \textbf{34.4} \\
\midrule
Llama-3.1-8B-Instruct & 0.2 & \textbf{0.7} & \textbf{1.2}  & \textbf{3.2} & \textbf{7.1} & \textbf{10.8} \\ 
Llama-3.1-8B-Instruct + GRPO & 0.2 & 0.3 & 0.5 & 2.1 & 3.0 & 7.0 \\
Llama-3.1-8B-Instruct + DAPO & \textbf{0.3} & 0.3 & 0.6 & 1.9 & 2.8 & 6.3  \\
Llama-3.1-8B-Instruct + ACE-GRPO & \textbf{0.3} & 0.3 & 0.5 & \underline{2.2} & \underline{3.9} & \underline{8.2} \\
Llama-3.1-8B-Instruct + ACE-DAPO & 0.2 & 0.3 & 0.6 & 2.0 & 3.2 & 7.1 \\
\bottomrule
\end{tabular}
\end{table}

\paragraph{Key findings (Qwen2.5-Math-7B).} As shown in Figure~\ref{fig:benchmark_results}, ACE consistently improves larger-$k$ metrics while maintaining comparable Pass@1. On MATH-500, ACE-GRPO improves Pass@32 from 91.3\% to 94.3\% (+3.0pp) over GRPO; ACE-DAPO further pushes Pass@32 to 96.1\% (+1.5pp over DAPO's 94.6\%). On AIME 2025, ACE-GRPO improves Pass@32 from 33.7\% to 36.4\% (+2.7pp); ACE-DAPO reaches 38.6\% (+1.5pp over DAPO's 37.1\%). Notably, ACE-DAPO achieves the strongest results across all $k$, demonstrating that ACE composes effectively with orthogonal diversity-preserving strategies.

\paragraph{Key findings (Qwen3-8B-Base).} The same pattern holds on a different model family. On MATH-500, ACE-GRPO improves Pass@32 from 88.6\% to 91.1\% (+2.5pp); ACE-DAPO reaches 91.6\% (+1.2pp over DAPO's 90.4\%). On AIME 2025, ACE-GRPO improves Pass@32 from 29.8\% to 32.4\% (+2.6pp); ACE-DAPO reaches 34.4\% (+1.3pp over DAPO's 33.1\%).

\paragraph{Key findings (Llama-3.1-8B-Instruct).} To test cross-family generalization, we evaluate on a non-Qwen model. On MATH-500, ACE-GRPO improves Pass@32 from 79.3\% to 81.5\% (+2.2pp); ACE-DAPO reaches 82.1\% (+1.7pp over DAPO's 80.4\%). On AIME 2025---where Llama-3.1-8B-Instruct operates near the floor---ACE-GRPO improves Pass@32 from 7.0\% to 8.2\% (+1.2pp), demonstrating that ACE's mechanism transfers across model families even under low-accuracy regimes.

These consistent gains across all three model families confirm the generality of ACE's mechanism.

\paragraph{Interaction with DAPO's Clip-Higher.}
A natural observation is that ACE's marginal gain over DAPO is smaller than 
over GRPO (e.g., on MATH-500 Qwen2.5-Math-7B Pass@32: +3.0pp for ACE-GRPO 
vs.\ GRPO, but +1.5pp for ACE-DAPO vs.\ DAPO). This reflects a genuine 
mechanism overlap: DAPO's Clip-Higher preserves diversity by limiting how 
aggressively \emph{any} incorrect path is suppressed at the token level, 
which indirectly reduces the overconfident-error pathology that ACE targets. 
However, DAPO's protection is \emph{indiscriminate}---it shields overconfident 
errors and exploratory errors alike, because token-level clipping cannot 
distinguish trajectory-level confidence regimes. ACE provides the missing 
selectivity: it amplifies suppression specifically for errors the model has 
learned to be confident in, while leaving exploratory errors untouched. 
The consistent gains of ACE-DAPO over DAPO across all model families and 
benchmarks indicate that this trajectory-level selectivity captures a 
dimension of the overconfidence problem that token-level clipping alone 
cannot resolve. The diminishing marginal returns are expected---both methods 
partially address the same pathology---but the residual improvement confirms 
that ACE's rollout-level discrimination provides value beyond what DAPO's 
uniform token-level mechanism achieves.

\begin{takeawaybox}[Main Result]
ACE preserves Pass@1 performance while significantly expanding the reasoning boundary at large $k$. Across three model families (Qwen2.5-Math-7B, Qwen3-8B-Base, Llama-3.1-8B-Instruct) and two benchmarks (MATH-500 and AIME 2025), ACE-GRPO consistently improves Pass@32 by +2.2--3.0pp over GRPO. Moreover, ACE composes with DAPO: ACE-DAPO achieves the best overall results (e.g., 96.1\% on MATH-500 Pass@32), improving over DAPO by +1.2--1.7pp, confirming that ACE provides complementary correction orthogonal to token-level diversity strategies and generalizes across model families.
\end{takeawaybox}

\begin{figure}[h]
\centering
\includegraphics[width=\textwidth]{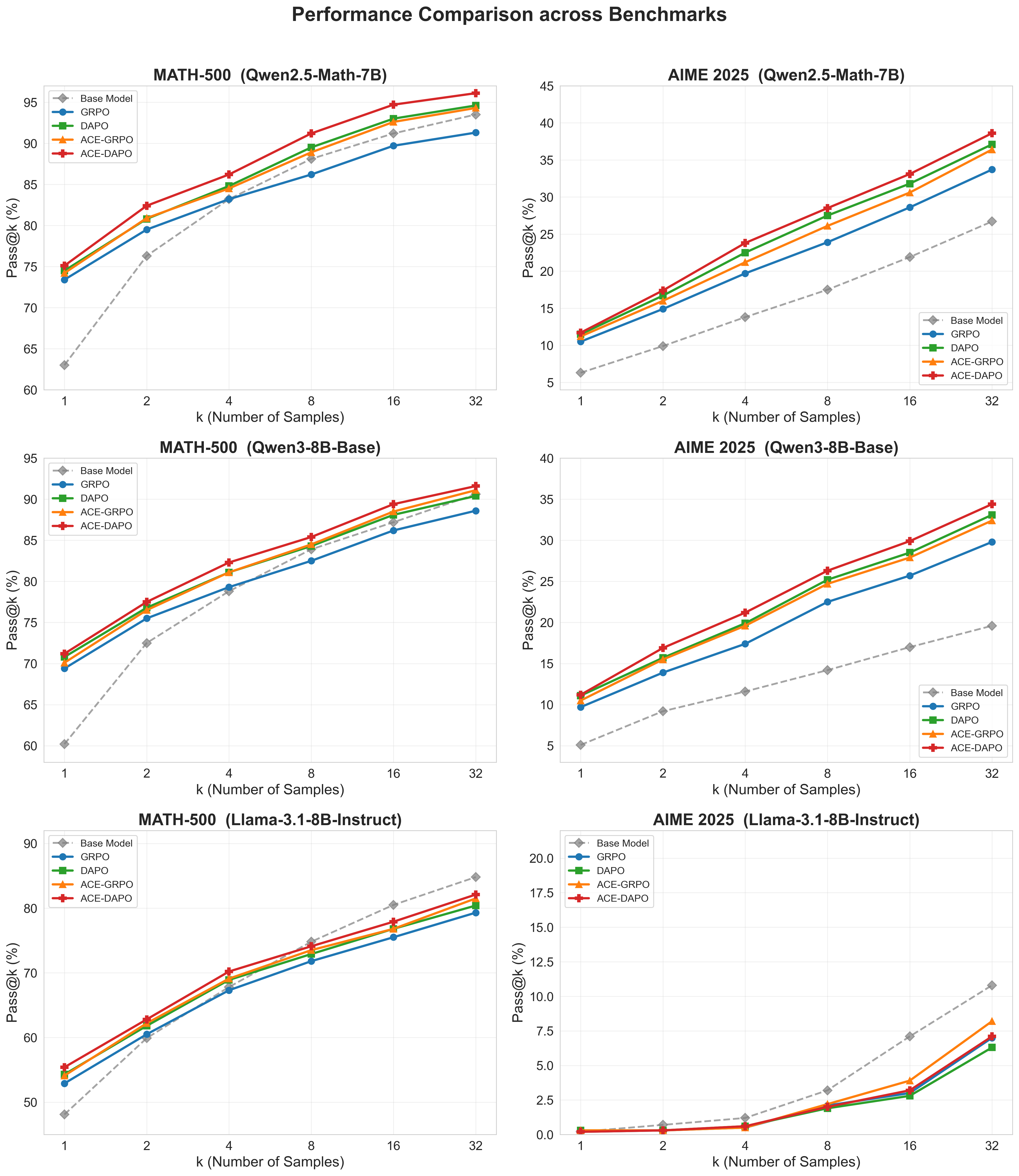}
\caption{\textbf{Performance Comparison across Benchmarks.} Pass@$k$ curves for all five methods on MATH-500 (left column) and AIME 2025 (right column) across three model families: Qwen2.5-Math-7B (top row), Qwen3-8B-Base (middle row), and Llama-3.1-8B-Instruct (bottom row). ACE-GRPO and ACE-DAPO consistently outperform their respective baselines (GRPO and DAPO) across all sampling budgets, model families, and benchmarks, with larger gains at higher $k$ values. ACE-DAPO achieves the best overall performance, confirming that ACE's rollout-level correction composes with DAPO's token-level diversity preservation and generalizes across model families.}
\label{fig:benchmark_results}
\end{figure}

\subsection{Experiment 1: Overconfident Error Dynamics}
\label{sec:exp_dynamics}

\paragraph{Goal.} Quantify the prevalence of overconfident errors during training and demonstrate that ACE effectively reduces them.

\paragraph{Design.} Track the distribution of $c_i$ among incorrect rollouts throughout training for both standard GRPO and ACE-GRPO on the two Qwen models.\footnote{We omit Llama-3.1-8B-Instruct from the diagnostic experiments as its lower baseline accuracy yields fewer correct rollouts per group, making the confidence shift statistics noisier and less informative. The main results in Tables~\ref{tab:math500_results}--\ref{tab:math_aime25_results} confirm that ACE's gains transfer to Llama.} At checkpoints every 25 training steps, generate 32 rollouts per prompt on a held-out set and record $c_i$ for all incorrect rollouts.

\paragraph{Metrics.}
\begin{itemize}[leftmargin=2em]
    \item \textbf{Overconfident error fraction}: $\text{OEF}(t) = |\{y_i \in \Ym : c_i > 0\}| / |\Ym|$ at step $t$.
    \item \textbf{Mean overconfidence magnitude}: $\E[c_i \mid c_i > 0, r_i = 0]$ at step $t$.
    \item \textbf{Token-level entropy}: Average per-token entropy of the policy.
\end{itemize}

\paragraph{Results.}
The core claim of ACE is that standard GRPO allows incorrect rollouts to become increasingly overconfident during training, and that ACE's asymmetric penalty should counteract this pathology. To test this, we track two complementary diagnostics at every checkpoint (Figure~\ref{fig:oef_dynamics}): (i) the \emph{overconfident error fraction} (OEF), which measures the proportion of incorrect rollouts whose confidence has grown relative to the reference policy ($c_i > 0$), and (ii) the \emph{mean overconfidence magnitude} among those overconfident errors, which captures the severity of the problem. Throughout training, ACE-GRPO maintains a lower OEF and a lower mean overconfidence magnitude than standard GRPO at every recorded checkpoint, indicating that ACE consistently suppresses both the prevalence and the severity of high-confidence incorrect rollouts.

\begin{figure}[h]
\centering
\includegraphics[width=\textwidth]{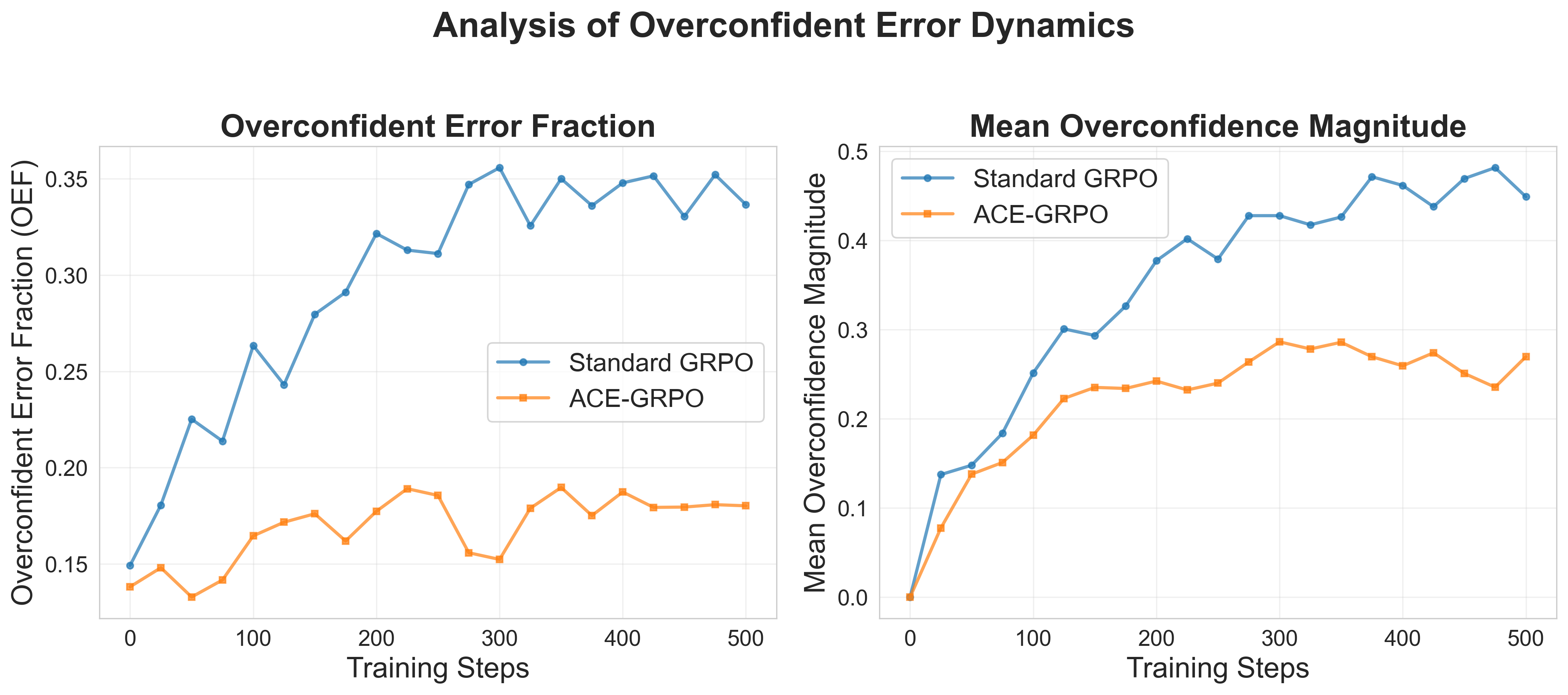}
\caption{\textbf{Overconfident Error Dynamics.} Left: Overconfident error fraction (OEF) over training. Right: Mean overconfidence magnitude for $c_i > 0$ errors. ACE-GRPO effectively suppresses both metrics compared to standard GRPO.}
\label{fig:oef_dynamics}
\end{figure}

\subsection{Experiment 2: Entropy Dynamics}
\label{sec:exp_entropy}

\paragraph{Goal.} Verify that ACE preserves generation diversity by tracking entropy throughout training, and establish the connection between entropy and Pass@$k$ performance.

\paragraph{Design.} Over the first 20 training steps, compute the average per-token entropy of the policy on a held-out subset of DAPO-Math-17K prompts:
\begin{equation}
    H(t) = -\frac{1}{|\mathcal{D}_{\text{val}}|} \sum_{x \in \mathcal{D}_{\text{val}}} \frac{1}{T} \sum_{j=1}^{T} \sum_{v} \pi_\theta(v | x, y_{<j}) \log \pi_\theta(v | x, y_{<j})
\end{equation}
where $T$ is the average sequence length and $v$ ranges over the vocabulary.

\paragraph{Results.}
A key concern with aggressive error suppression is that it may cause premature mode collapse, concentrating probability mass on a narrow set of outputs and destroying the diversity needed for high Pass@$k$ at large $k$. To diagnose this, we track average per-token entropy $H(t)$ over the early phase of training, where entropy decay is most rapid, for both standard GRPO and ACE-GRPO (Figure~\ref{fig:entropy_dynamics}). Standard GRPO exhibits a sharp entropy drop within the first 20 steps, retaining only a small fraction of its initial entropy. In contrast, ACE-GRPO decays substantially more slowly, preserving a much larger fraction of the initial entropy over the same period. This gap correlates with Pass@$k$ performance at large $k$: the method that retains more entropy also achieves higher coverage, confirming that ACE's selective penalty avoids premature mode collapse while still suppressing overconfident errors.

\begin{figure}[h]
\centering
\includegraphics[width=\textwidth]{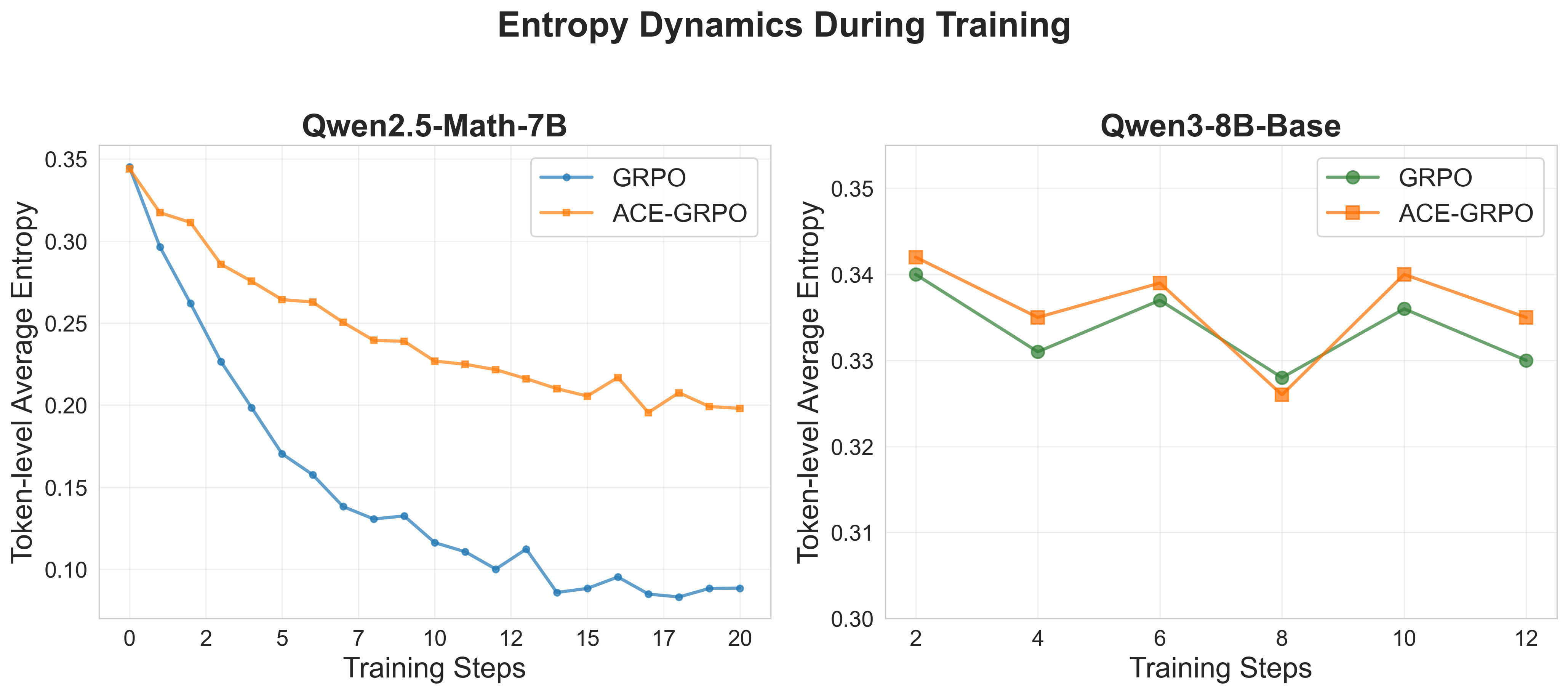}
\caption{\textbf{Entropy Dynamics.} Token-level entropy over the first 20 training steps. Left: On Qwen2.5-Math-7B, ACE-GRPO retains substantially more entropy than standard GRPO, which suffers rapid entropy collapse. Right: On Qwen3-8B-Base, ACE-GRPO maintains more stable entropy, demonstrating consistency across architectures. We report entropy dynamics for the two Qwen models only; Llama-3.1-8B-Instruct is excluded because its lower baseline accuracy makes the entropy signal less directly comparable (see \S\ref{sec:experiments} for discussion).}
\label{fig:entropy_dynamics}
\end{figure}

\subsection{Ablation: Choice of Modulation Function}
\label{sec:ablation_modulation}

A natural question is whether the choice of $\Softplus$ as the modulation function is important, or whether a simpler alternative such as $\text{ReLU}(c_i) = \max(0, c_i)$ suffices. We compare the two variants on MATH-500 using Qwen2.5-Math-7B with $\alpha = 1.0$ (the ablation uses a single representative model to isolate the effect of the modulation function; the main results in Table~\ref{tab:math500_results} confirm that ACE's gains are consistent across all three model families):
\begin{itemize}[leftmargin=2em]
    \item \textbf{ACE-Softplus} (default): $A_{\mathrm{ACE},i}^- = \hat{A}_i^- \cdot (1 + \alpha \cdot \Softplus(c_i))$
    \item \textbf{ACE-ReLU}: $A_{\mathrm{ACE},i}^- = \hat{A}_i^- \cdot (1 + \alpha \cdot \text{ReLU}(c_i))$
\end{itemize}
ReLU completely ignores self-correcting and exploratory errors ($c_i \leq 0$), providing zero modulation in that regime, while Softplus provides a smooth, everywhere-positive modulation that transitions gradually.

\begin{table}[h]
\centering
\caption{Ablation: modulation function on MATH-500 (Qwen2.5-Math-7B, $\alpha = 1.0$).}
\label{tab:ablation_modulation}
\begin{tabular}{lcccccc}
\toprule
\textbf{Method} & \textbf{@1} & \textbf{@2} & \textbf{@4} & \textbf{@8} & \textbf{@16} & \textbf{@32} \\
\midrule
GRPO (baseline) & 73.4 & 79.5 & 83.2 & 86.2 & 89.7 & 91.3 \\
ACE-ReLU & 73.2 & 80.3 & 83.9 & 87.6 & 91.2 & 93.1 \\
ACE-Softplus (ours) & 74.2 & 80.9 & 84.5 & 88.9 & 92.6 & 94.3 \\
\bottomrule
\end{tabular}
\end{table}

\paragraph{Analysis.} Both ACE-ReLU and ACE-Softplus outperform standard GRPO across all $k > 1$, confirming that confidence-aware modulation---regardless of the specific activation---is beneficial. However, ACE-Softplus consistently outperforms ACE-ReLU, with the gap widening at larger $k$ (+1.2 pp at Pass@32). This advantage stems from two properties of Softplus. First, \emph{smoothness}: ReLU has a non-differentiable kink at $c_i = 0$, creating a discontinuity in the gradient landscape that can destabilize training, whereas Softplus provides smooth gradient flow everywhere. Second, \emph{non-zero modulation near the boundary}: ReLU assigns zero modulation to all errors with $c_i \leq 0$, treating them identically to standard GRPO. In contrast, Softplus($0$) $= \ln 2 \approx 0.69$, providing a gentle baseline modulation that enables finer differentiation among borderline errors near $c_i \approx 0$---precisely the regime where errors may be transitioning from exploratory to overconfident. These results empirically validate the design rationale in \S\ref{sec:ace_advantage}.

\subsection{Analysis: Mechanism Behind Diversity Preservation}
\label{sec:entropy}

The experimental results above (\S\ref{sec:exp_dynamics}--\S\ref{sec:exp_entropy}) reveal a consistent mechanism: standard GRPO's uniform penalties allow overconfident errors to form ``probability sinks'' that crowd out valid reasoning paths---the pathology identified by \citet{yue2025limit} as the root cause of RLVR's narrowing reasoning boundary. ACE's asymmetric penalties break this cycle: the selective KL term (Theorem~\ref{thm:selective_kl}) acts as entropy regularization restricted to the overconfident region, while leaving exploratory errors ($c_i \leq 0$) untouched. This targeted correction redistributes probability mass to alternative reasoning paths, explaining ACE's improvements across the full Pass@$k$ spectrum.

\section{Limitations and Future Work}
\label{sec:limitations}

\paragraph{Dependence on reference model quality.} ACE uses $\piref$ to define overconfidence. If the reference model is poorly calibrated, the confidence score $c_i$ may not reliably indicate spurious patterns. Exploring alternatives (e.g., using a moving average of recent checkpoints) is a direction for future work.

\paragraph{Binary rewards only.} Our current formulation assumes binary rewards ($r \in \{0, 1\}$). Extending ACE to continuous or partial rewards (e.g., from process reward models) requires redefining what constitutes an ``overconfident error'' in the presence of graded feedback.

\paragraph{Interaction with long CoT.} Extended reasoning models (e.g., with $>$10K token outputs) may exhibit different confidence shift dynamics. The sequence-length normalization ($\bar{c}_i = c_i / T_i$) may need refinement for very long chains.

\section{Conclusion}
\label{sec:conclusion}

We identified a previously overlooked pathology in RLVR training: the accumulation of overconfident errors---incorrect reasoning paths that the RL process spuriously reinforces. We proposed ACE, a simple modification to the advantage function that dynamically amplifies penalties for overconfident errors while leaving exploratory errors untouched.

\begin{takeawaybox}[Core Contributions]
\textbf{(1)} We formalize \emph{error confidence shift} as a new per-rollout diagnostic orthogonal to prompt difficulty, revealing that overconfident errors accumulate during RLVR and drive diversity collapse.
\textbf{(2)} ACE's gradient decomposes into a \emph{selective regularizer} on overconfident errors plus a tempering residual, providing principled theoretical grounding.
\textbf{(3)} ACE improves the full Pass@$k$ spectrum---especially at large $k$---without sacrificing Pass@1, adding only a single Softplus computation per incorrect rollout.
\end{takeawaybox}

\bibliographystyle{unsrtnat}
\bibliography{references}

\appendix

\section{Proof of Theorem~\ref{thm:selective_kl} (Selective Regularization Decomposition)}
\label{app:selective_kl_proof}

\begin{proof}
The gradient of $\mathcal{L}_{\mathrm{ACE}}$ differs from $\mathcal{L}_{\mathrm{std}}$ only in the negative advantage terms. We analyze the $\alpha$-dependent component. Since $A_{\mathrm{ACE},i}^- = \hat{A}_i^- \cdot (1 + \alpha \cdot \Softplus(c_i))$, the additional gradient relative to standard GRPO is:
\begin{equation}
    \Delta \nabla_\theta = -\alpha \E_{x \sim \cD}\left[\frac{|\hat{A}^-(x)|}{G}\sum_{y_i \in \Ym(x)} \Softplus(c_i) \cdot \nabla_\theta \log \pith(y_i|x)\right]
\end{equation}
Here $|\hat{A}^-(x)|$ is a per-prompt scalar (constant across rollouts within a group) that does not depend on $y_i$. This is the standard REINFORCE form: $|\hat{A}^-(x)| \cdot \Softplus(c_i)$ acts as a \emph{scalar reward} multiplying the score function, with $c_i$ treated as not depending on $\theta$ (the ``stop-gradient'' convention standard in policy gradient methods).

As $G \to \infty$, by the law of large numbers:
\begin{align}
    \Delta \nabla_\theta &\to -\alpha \E_{x \sim \cD}\left[|\hat{A}^-(x)| \sum_{y \in \Ym(x)} \pith(y|x) \cdot \Softplus(c(y)) \cdot \nabla_\theta \log \pith(y|x)\right] \nonumber \\
    &= -\alpha \E_{x \sim \cD}\left[|\hat{A}^-(x)| \sum_{y \in \Ym(x)} \Softplus(c(y)) \cdot \nabla_\theta \pith(y|x)\right] \label{eq:ace_gradient}
\end{align}
using $\pith(y|x) \nabla_\theta \log \pith(y|x) = \nabla_\theta \pith(y|x)$.

Now, the true gradient of $\mathcal{R}_{\mathrm{sel}}(\theta)$ requires differentiating $|\hat{A}^-(x)| \cdot \pith(y|x) \cdot \Softplus(c(y))$ where $\pith(y|x)$ and $\Softplus(c(y))$ both depend on $\theta$ (since $c(y) = \log \pith(y|x) - \log \piref(y|x)$). Since $|\hat{A}^-(x)|$ is a per-prompt scalar, it factors out, and by the product rule:
\begin{align}
    \nabla_\theta \left[\pith(y|x) \cdot \Softplus(c)\right] &= \underbrace{\Softplus(c) \cdot \nabla_\theta \pith(y|x)}_{\text{Term~I: captured by ACE}} + \underbrace{\pith(y|x) \cdot \sigma(c) \cdot \nabla_\theta \log \pith(y|x)}_{\text{Term~II: residual}} \label{eq:product_rule}
\end{align}
where $\sigma(c) = \Softplus'(c) = 1/(1 + e^{-c})$ and $\nabla_\theta c = \nabla_\theta \log \pith(y|x)$.

Multiplying by $|\hat{A}^-(x)|$, summing over $y \in \Ym(x)$, and taking expectations, Eq.~\eqref{eq:ace_gradient} matches exactly $-\alpha \cdot |\hat{A}^-(x)| \cdot \text{Term~I}$. Rearranging:
\begin{equation}
    \Delta \nabla_\theta = -\alpha \nabla_\theta \mathcal{R}_{\mathrm{sel}} + \alpha \underbrace{\E_{x \sim \cD}\!\left[|\hat{A}^-(x)| \sum_{y \in \Ym(x)} \pith(y|x)\, \sigma(c)\, \nabla_\theta \log \pith(y|x)\right]}_{\mathcal{E}(\theta)}
\end{equation}
This is an \emph{exact} identity with no approximation. The residual $\mathcal{E}(\theta)$ arises because ACE treats $\Softplus(c_i)$ as a fixed reward signal, omitting the gradient through $c_i$ itself.
\end{proof}

\begin{remark}[Residual term and contrast with global KL]
\label{remark:residual_magnitude}
The residual $\mathcal{E}(\theta)$ is not negligible: for $c \in [1, 3]$, the ratio $\sigma(c)/\Softplus(c)$ ranges from 31--56\%. $\mathcal{E}$ arises because ACE treats $\Softplus(c_i)$ as a fixed scalar (stop-gradient), omitting the gradient that the full regularizer $\mathcal{R}_{\mathrm{sel}}$ would contribute by differentiating \emph{through} $c_i$ (Term~II in Eq.~\ref{eq:product_rule}). This omitted gradient would suppress overconfident errors \emph{more aggressively}: it drives the parameters to reduce not only $\pi_\theta(y|x)$ but also the confidence gap $c(y)$ itself. ACE therefore implements a \emph{tempered} version of the full regularizer---correcting overconfident errors via the dominant Term~I (confidence-weighted probability suppression) while forgoing Term~II's sharper through-$c$ correction.
\end{remark}

\begin{remark}[Why stop-gradient is preferable to the full regularizer]
\label{remark:why_stop_gradient}
A natural question is whether one should retain Term~II to implement the full $\nabla_\theta \mathcal{R}_{\mathrm{sel}}$ instead of ACE's tempered version. We argue against this for three reasons. \textbf{(i)~Precedent for detaching $\theta$-dependent signals.} Although the reward in vanilla REINFORCE does not depend on $\theta$, modern policy gradient methods routinely stop-gradient through $\theta$-dependent quantities used in the loss: PPO/GRPO detach the advantage $\hat{A}_i$ (computed from the current policy's rollouts) from the actor gradient; actor-critic methods detach the value baseline $V(s;\theta)$ even under parameter sharing; and the ``old policy'' $\piold$ in importance ratios is always frozen. ACE's treatment of $\Softplus(c_i)$ as a detached reward modifier follows the same principle: quantities that \emph{diagnose} the policy state should inform gradient \emph{magnitude}, not become optimization targets themselves. \textbf{(ii)~Feedback loop.} Retaining Term~II means the penalty magnitude itself becomes an optimization target: the gradient would simultaneously try to reduce $\pi_\theta(y|x)$ \emph{and} reduce $c_i = \log(\pith/\piref)$, creating a second-order feedback that can cause gradient oscillation and training instability. \textbf{(iii)~Variance.} The gradient quality analysis (Theorem~\ref{thm:variance_bound}) proves that ACE's stop-gradient version improves the quality ratio $Q_d$ under realistic conditions. Adding Term~II introduces an additional score-function estimator $\sigma(c_i) \nabla_\theta \log \pith$, which increases gradient variance without a guaranteed commensurate signal gain---the sufficient condition for quality improvement (Eq.~\ref{eq:gamma_condition}) would need to be re-derived and may no longer hold.
\end{remark}

\section{Gradient Quality Analysis}
\label{app:gradient_quality}

This appendix provides the full formal analysis of ACE's effect on gradient quality, summarized in \S\ref{sec:variance_analysis}.

\begin{assumption}
\label{assumption:gradient}
For a fixed prompt $x$ with pass rate $p$, let $g_i = A_i \nabla_\theta \log \pith(y_i|x)$ be the per-rollout gradient for incorrect rollouts ($r_i = 0$), and let $s_i = \nabla_\theta \log \pith(y_i|x)$ denote the score function. Let $\phi_i = \Softplus(c_i)$. We assume:
\begin{enumerate}[leftmargin=2em]
    \item Rollouts $y_i$ are conditionally independent given $x$.
    \item The signal direction is $\hat{d} = \E[s_i \mid r_i=0] / \|\E[s_i \mid r_i=0]\|$.
    \item The directional covariance satisfies $\mathrm{Cov}(\phi_i,\, (\hat{d}^\top s_i)^2 \mid r_i = 0) > 0$, i.e., overconfident errors tend to have score functions more aligned with the expected gradient direction.
\end{enumerate}
\end{assumption}

\begin{proposition}[Second Moment Increase]
\label{prop:second_moment}
For any $\alpha > 0$, ACE strictly increases the mean squared gradient norm of incorrect rollouts:
\begin{equation}
    \E[\|g_i^{\mathrm{ACE}}\|^2 \mid r_i = 0] > \E[\|g_i^{\mathrm{std}}\|^2 \mid r_i = 0]
\end{equation}
whenever $\E[\phi_i \|s_i\|^2 \mid r_i = 0] > 0$ (i.e., errors are not all zero-gradient). This is an unavoidable consequence of the purely additive penalty structure: $(1 + \alpha \phi_i) > 1$ for all $\phi_i > 0$.
\end{proposition}

\begin{proof}
Let $a = |\hat{A}^-(x)| > 0$ denote the per-prompt base penalty magnitude. Under standard GRPO: $g_i^{\mathrm{std}} = a \cdot s_i$. Under ACE: $g_i^{\mathrm{ACE}} = a(1 + \alpha\phi_i) \cdot s_i$. Then:
\begin{align}
    \E[\|g_i^{\mathrm{ACE}}\|^2] - \E[\|g_i^{\mathrm{std}}\|^2] &= a^2\left(\E[(1+\alpha\phi_i)^2 \|s_i\|^2] - \E[\|s_i\|^2]\right) \nonumber \\
    &= a^2\left(\underbrace{2\alpha\,\E[\phi_i \|s_i\|^2]}_{> 0} + \underbrace{\alpha^2\,\E[\phi_i^2 \|s_i\|^2]}_{\geq 0}\right) > 0
\end{align}
since $a > 0$, $\phi_i = \Softplus(c_i) > 0$, $\alpha > 0$, and $\|s_i\|^2 \geq 0$ with $\E[\phi_i \|s_i\|^2] > 0$.
\end{proof}

\begin{definition}[Directional Signal and Variance]
\label{def:variance_decomposition}
For incorrect rollouts, let $\hat{d} = \E[s_i \mid r_i=0] / \|\E[s_i \mid r_i=0]\|$ be the unit vector along the expected score function. The \textbf{directional signal} and \textbf{directional variance} of a gradient estimator $g_i = w_i \cdot s_i$ are:
\begin{align}
    \mu_d &= \E[\hat{d}^\top g_i \mid r_i = 0] && \text{(signal along $\hat{d}$)} \\
    \sigma_d^2 &= \mathrm{Var}[\hat{d}^\top g_i \mid r_i = 0] && \text{(noise along $\hat{d}$)}
\end{align}
The \textbf{gradient quality ratio} is $Q_d = \mu_d^2 / \sigma_d^2$.
\end{definition}

\begin{theorem}[Improved Gradient Quality via ACE]
\label{thm:variance_bound}
Under Assumption~\ref{assumption:gradient}, let $\hat{d}$ be the signal direction. Define the directional projections $u_i = \hat{d}^\top s_i$ (scalar random variables). Assume:
\begin{equation}
    \mathrm{Cov}(\phi_i,\, u_i^2 \mid r_i = 0) > 0
\end{equation}
i.e., overconfident errors tend to have score functions more aligned with the expected gradient direction. Then:

\medskip
\noindent\textbf{(a) Directional variance increase.} For any $\alpha > 0$, ACE increases the directional variance:
\begin{equation}
\label{eq:dir_var_increase}
    \mathrm{Var}[\hat{d}^\top g_i^{\mathrm{ACE}} \mid r_i = 0] > \mathrm{Var}[\hat{d}^\top g_i^{\mathrm{std}} \mid r_i = 0]
\end{equation}
whenever $\mathrm{Cov}(\phi_i, u_i^2) > 0$ and $\E[\phi_i] > 0$. This is an unavoidable consequence of the additive reweighting structure, analogous to the total second-moment increase (Proposition~\ref{prop:second_moment}).

\noindent\textbf{(b) Quality improvement under high-variance conditions.} Assume additionally that the initial gradient is noisy relative to the signal, i.e., $\mathrm{Var}[u_i] > (\E[u_i])^2$ (equivalently, $Q_d^{\mathrm{std}} < 1$). Then for sufficiently small $\alpha > 0$, the gradient quality ratio of ACE strictly dominates that of standard GRPO:
\begin{equation}
\label{eq:quality_improvement}
    Q_d^{\mathrm{ACE}} > Q_d^{\mathrm{std}}
\end{equation}

Consequently, although ACE increases both the signal and the noise (directional variance), the signal grows faster, yielding a net improvement in gradient quality along the optimization-relevant direction.
\end{theorem}

\begin{proof}
Consider a fixed prompt $x$ with $G$ rollouts sampled i.i.d.\ from $\pith(\cdot|x)$. Let $s_i = \nabla_\theta \log \pith(y_i|x)$ denote the score function for rollout $y_i$, and let $\phi_i = \Softplus(c_i)$. All expectations below are conditioned on $r_i = 0$.

\paragraph{Setup and notation.} Let $a = |\hat{A}^-(x)| > 0$ denote the per-prompt base penalty magnitude. Under standard GRPO: $g_i^{\mathrm{std}} = a \cdot s_i$. Under ACE: $g_i^{\mathrm{ACE}} = a(1 + \alpha \phi_i) s_i$. Since $a$ is a positive scalar constant (per-prompt), it cancels in the gradient quality ratio $Q_d = \mu_d^2/\sigma_d^2$. We therefore analyze the normalized weights $w_i^{\mathrm{std}} = 1$ and $w_i^{\mathrm{ACE}} = 1 + \alpha\phi_i$ without loss of generality. Let $\hat{d} = \E[s_i] / \|\E[s_i]\|$ be the signal direction, and define the scalar projections $u_i = \hat{d}^\top s_i$.

\paragraph{Step 1: Directional variance analysis (Part (a)).} The directional variance is:
\begin{equation}
    \mathrm{Var}[w_i u_i] = \E[w_i^2 u_i^2] - (\E[w_i u_i])^2
\end{equation}
For standard GRPO ($w_i = 1$): $\mathrm{Var}^{\mathrm{std}} = \E[u_i^2] - (\E[u_i])^2$.

For ACE ($w_i = 1 + \alpha \phi_i$):
\begin{align}
    \E[w_i^2 u_i^2] &= \E[u_i^2] + 2\alpha\,\E[\phi_i u_i^2] + \alpha^2\,\E[\phi_i^2 u_i^2] \label{eq:app_second} \\
    (\E[w_i u_i])^2 &= (\E[u_i] + \alpha\,\E[\phi_i u_i])^2 \nonumber \\
    &= (\E[u_i])^2 + 2\alpha\,\E[u_i]\,\E[\phi_i u_i] + \alpha^2 (\E[\phi_i u_i])^2 \label{eq:app_signal}
\end{align}

Subtracting Eq.~\eqref{eq:app_signal} from Eq.~\eqref{eq:app_second}:
\begin{align}
    \mathrm{Var}^{\mathrm{ACE}} &= \mathrm{Var}^{\mathrm{std}} + 2\alpha\underbrace{\left(\E[\phi_i u_i^2] - \E[u_i]\,\E[\phi_i u_i]\right)}_{\equiv\, \Delta_1} + O(\alpha^2) \label{eq:app_var_diff}
\end{align}

\paragraph{Step 2: Sign of $\Delta_1$.} Decompose using identities:
\begin{align}
    \E[\phi_i u_i^2] &= \mathrm{Cov}(\phi_i, u_i^2) + \E[\phi_i]\,\E[u_i^2] \\
    \E[\phi_i u_i] &= \mathrm{Cov}(\phi_i, u_i) + \E[\phi_i]\,\E[u_i]
\end{align}
Substituting into $\Delta_1$:
\begin{align}
    \Delta_1 &= \mathrm{Cov}(\phi_i, u_i^2) + \E[\phi_i]\,\E[u_i^2] - \E[u_i]\left(\mathrm{Cov}(\phi_i, u_i) + \E[\phi_i]\,\E[u_i]\right) \nonumber \\
    &= \mathrm{Cov}(\phi_i, u_i^2) - \E[u_i]\,\mathrm{Cov}(\phi_i, u_i) + \E[\phi_i]\,\mathrm{Var}[u_i] \label{eq:delta1}
\end{align}

The third term $\E[\phi_i]\,\mathrm{Var}[u_i] > 0$ always increases variance. In fact, $\Delta_1 > 0$ under typical conditions: for the natural Gaussian linear model where $u_i \sim \mathcal{N}(\mu, \sigma^2)$ and $\phi_i = a + bu_i$ ($a > 0$, $b > 0$), we have $\mathrm{Cov}(\phi_i, u_i^2) = 2b\mu\sigma^2$, $\mathrm{Cov}(\phi_i, u_i) = b\sigma^2$, and $\E[\phi_i] = a + b\mu$, giving:
\begin{equation}
    \Delta_1 = 2b\mu\sigma^2 - \mu \cdot b\sigma^2 + (a + b\mu)\sigma^2 = 2b\mu\sigma^2 + a\sigma^2 > 0
\end{equation}
This confirms that \textbf{directional variance increases} under ACE---an unavoidable cost of additive reweighting. The condition $\Delta_1 < 0$ would require:
\begin{equation}
    \E[u_i]\,\mathrm{Cov}(\phi_i, u_i) > \mathrm{Cov}(\phi_i, u_i^2) + \E[\phi_i]\,\mathrm{Var}[u_i] \label{eq:dir_var_cond}
\end{equation}
which is violated in the Gaussian linear model and is difficult to satisfy in practice. However, as we show next, this does \emph{not} prevent quality improvement: what matters is that the signal grows faster than the square root of variance.

\paragraph{Step 3: Quality improvement (Part (b)).} The gradient quality ratio is:
\begin{equation}
    Q_d = \frac{(\E[w_i u_i])^2}{\mathrm{Var}[w_i u_i]}
\end{equation}

Since $\Delta_1 > 0$ in general (Step~2), the directional variance increases. Nevertheless, the quality ratio can still improve because ACE also increases the signal $\E[w_i u_i]$. We now provide the complete derivation.

\textbf{Signal computation.} For ACE with $w_i = 1 + \alpha \phi_i$:
\begin{align}
    \mu^{\mathrm{ACE}} &\triangleq \E[w_i u_i] = \E[(1 + \alpha \phi_i) u_i] = \E[u_i] + \alpha\,\E[\phi_i u_i] \nonumber \\
    &= \E[u_i] + \alpha\left(\mathrm{Cov}(\phi_i, u_i) + \E[\phi_i]\,\E[u_i]\right) \nonumber \\
    &= \E[u_i](1 + \alpha\,\E[\phi_i]) + \alpha\,\mathrm{Cov}(\phi_i, u_i)
\end{align}
Denote $\mu \triangleq \E[u_i]$, $\bar{\phi} \triangleq \E[\phi_i]$, and $C \triangleq \mathrm{Cov}(\phi_i, u_i)$. Then:
\begin{equation}
    \mu^{\mathrm{ACE}} = \mu(1 + \alpha\bar{\phi}) + \alpha C
\end{equation}
The squared signal is:
\begin{align}
    (\mu^{\mathrm{ACE}})^2 &= \left(\mu(1 + \alpha\bar{\phi}) + \alpha C\right)^2 \nonumber \\
    &= \mu^2(1 + \alpha\bar{\phi})^2 + 2\alpha C \mu(1 + \alpha\bar{\phi}) + \alpha^2 C^2 \nonumber \\
    &= \mu^2 + 2\alpha\mu^2\bar{\phi} + \alpha^2\mu^2\bar{\phi}^2 + 2\alpha C\mu + 2\alpha^2 C\mu\bar{\phi} + \alpha^2 C^2 \nonumber \\
    &= \mu^2 + 2\alpha(\mu^2\bar{\phi} + C\mu) + O(\alpha^2) \label{eq:signal_squared}
\end{align}

\textbf{Variance computation.} From Step~2, we have:
\begin{equation}
    \mathrm{Var}^{\mathrm{ACE}} = \mathrm{Var}^{\mathrm{std}} + 2\alpha\,\Delta_1 + O(\alpha^2) \label{eq:var_ace}
\end{equation}
where $\mathrm{Var}^{\mathrm{std}} = \E[u_i^2] - \mu^2$ and $\Delta_1$ is given by Eq.~\eqref{eq:delta1}.

\textbf{Quality ratio expansion.} We compute the difference in quality ratios. For standard GRPO:
\begin{equation}
    Q_d^{\mathrm{std}} = \frac{\mu^2}{\mathrm{Var}^{\mathrm{std}}}
\end{equation}
For ACE, using Eqs.~\eqref{eq:signal_squared} and \eqref{eq:var_ace}:
\begin{align}
    Q_d^{\mathrm{ACE}} &= \frac{(\mu^{\mathrm{ACE}})^2}{\mathrm{Var}^{\mathrm{ACE}}} = \frac{\mu^2 + 2\alpha(\mu^2\bar{\phi} + C\mu) + O(\alpha^2)}{\mathrm{Var}^{\mathrm{std}} + 2\alpha\,\Delta_1 + O(\alpha^2)}
\end{align}
Using the first-order Taylor expansion $(1 + x)^{-1} \approx 1 - x$ for small $x$:
\begin{align}
    Q_d^{\mathrm{ACE}} &= \frac{\mu^2 + 2\alpha(\mu^2\bar{\phi} + C\mu)}{\mathrm{Var}^{\mathrm{std}}}\left(1 - \frac{2\alpha\,\Delta_1}{\mathrm{Var}^{\mathrm{std}}}\right) + O(\alpha^2) \nonumber \\
    &= \frac{\mu^2}{\mathrm{Var}^{\mathrm{std}}} + \frac{2\alpha(\mu^2\bar{\phi} + C\mu)}{\mathrm{Var}^{\mathrm{std}}} - \frac{2\alpha\mu^2\,\Delta_1}{(\mathrm{Var}^{\mathrm{std}})^2} + O(\alpha^2) \nonumber \\
    &= Q_d^{\mathrm{std}} + \frac{2\alpha}{\mathrm{Var}^{\mathrm{std}}}\left(\mu^2\bar{\phi} + C\mu - \frac{\mu^2}{\mathrm{Var}^{\mathrm{std}}}\,\Delta_1\right) + O(\alpha^2)
\end{align}
Therefore:
\begin{equation}
    Q_d^{\mathrm{ACE}} - Q_d^{\mathrm{std}} = \frac{2\alpha}{\mathrm{Var}^{\mathrm{std}}}\underbrace{\left(\mu^2\bar{\phi} + C\mu - Q_d^{\mathrm{std}}\,\Delta_1\right)}_{\triangleq\, \Gamma} + O(\alpha^2) \label{eq:quality_diff}
\end{equation}

\textbf{Sufficient condition for improvement.} Quality improves when $\Gamma > 0$. Substituting $\Delta_1$ from Eq.~\eqref{eq:delta1}:
\begin{align}
    \Gamma &= \mu^2\bar{\phi} + C\mu - Q_d^{\mathrm{std}}\left(\mathrm{Cov}(\phi_i, u_i^2) - \mu\,C + \bar{\phi}\,\mathrm{Var}^{\mathrm{std}}\right) \nonumber \\
    &= \mu^2\bar{\phi} + C\mu - Q_d^{\mathrm{std}}\,\mathrm{Cov}(\phi_i, u_i^2) + Q_d^{\mathrm{std}}\mu\,C - Q_d^{\mathrm{std}}\bar{\phi}\,\mathrm{Var}^{\mathrm{std}}
\end{align}
Using $Q_d^{\mathrm{std}} = \mu^2/\mathrm{Var}^{\mathrm{std}}$, the last term becomes $-\mu^2\bar{\phi}$, which cancels with the first term:
\begin{align}
    \Gamma &= C\mu + Q_d^{\mathrm{std}}\mu\,C - Q_d^{\mathrm{std}}\,\mathrm{Cov}(\phi_i, u_i^2) \nonumber \\
    &= C\mu(1 + Q_d^{\mathrm{std}}) - Q_d^{\mathrm{std}}\,\mathrm{Cov}(\phi_i, u_i^2) \label{eq:gamma_simplified}
\end{align}
Under Assumption~\ref{assumption:gradient}, $C = \mathrm{Cov}(\phi_i, u_i) > 0$ and $\mathrm{Cov}(\phi_i, u_i^2) > 0$ (overconfident errors have gradients more aligned with the signal direction). We analyze $\Gamma > 0$:
\begin{equation}
    \Gamma > 0 \quad\Longleftrightarrow\quad C\mu(1 + Q_d^{\mathrm{std}}) > Q_d^{\mathrm{std}}\,\mathrm{Cov}(\phi_i, u_i^2)
\end{equation}
Rearranging:
\begin{equation}
    \frac{C\mu}{\mathrm{Cov}(\phi_i, u_i^2)} > \frac{Q_d^{\mathrm{std}}}{1 + Q_d^{\mathrm{std}}} \label{eq:gamma_condition}
\end{equation}
The right-hand side is a monotonically increasing function of $Q_d^{\mathrm{std}}$ that ranges from $0$ (when $Q_d^{\mathrm{std}} = 0$) to $1$ (as $Q_d^{\mathrm{std}} \to \infty$). Therefore, when $Q_d^{\mathrm{std}} < 1$ (equivalently, $\mathrm{Var}[u_i] > (\E[u_i])^2$), we have:
\begin{equation}
    \frac{Q_d^{\mathrm{std}}}{1 + Q_d^{\mathrm{std}}} < \frac{1}{2}
\end{equation}
Under the Gaussian linear model ($u_i \sim \mathcal{N}(\mu, \sigma^2)$, $\phi_i = a + bu_i$), we can verify:
\begin{align}
    C &= \mathrm{Cov}(\phi_i, u_i) = b\sigma^2 \\
    \mathrm{Cov}(\phi_i, u_i^2) &= b\,\mathrm{Cov}(u_i, u_i^2) = b \cdot 2\mu\sigma^2 = 2b\mu\sigma^2
\end{align}
Thus:
\begin{equation}
    \frac{C\mu}{\mathrm{Cov}(\phi_i, u_i^2)} = \frac{b\sigma^2 \cdot \mu}{2b\mu\sigma^2} = \frac{1}{2}
\end{equation}
Combined with Eq.~\eqref{eq:gamma_condition}, when $Q_d^{\mathrm{std}} < 1$:
\begin{equation}
    \frac{1}{2} > \frac{Q_d^{\mathrm{std}}}{1 + Q_d^{\mathrm{std}}} \quad\Longrightarrow\quad \Gamma > 0 \quad\Longrightarrow\quad Q_d^{\mathrm{ACE}} > Q_d^{\mathrm{std}}
\end{equation}
This completes the proof that quality improves under the high-variance condition:
\begin{equation}
    \mathrm{Var}[u_i] > (\E[u_i])^2 \qquad \Longleftrightarrow \qquad Q_d^{\mathrm{std}} < 1 \label{eq:high_var_cond}
\end{equation}
This high-variance regime is the typical operating condition in stochastic policy gradient optimization, where individual rollout gradients are highly variable. Under this condition, the signal growth term dominates the variance growth term, ensuring $Q_d^{\mathrm{ACE}} > Q_d^{\mathrm{std}}$.

\paragraph{Summary.} ACE's confidence-dependent weighting increases both the total gradient second moment (Proposition~\ref{prop:second_moment}) and the directional variance (Steps~1--2)---both unavoidable consequences of additive reweighting. However, ACE improves the gradient quality ratio (Step~3) under two conditions: (i) overconfident errors carry gradient signal aligned with the optimization direction ($\mathrm{Cov}(\phi_i, u_i) > 0$), and (ii) the initial quality ratio is low ($\mathrm{Var}[u_i] > (\E[u_i])^2$). The key mechanism is that ACE's selective amplification of high-confidence errors concentrates extra weight on the most informative gradients, causing the signal to grow faster than the noise along the optimization-relevant direction.
\end{proof}

\section{Implementation Details}
\label{app:implementation}

\paragraph{Sequence-level vs.\ token-level aggregation.} While Definition~\ref{def:confidence} defines $c_i$ at the sequence level, one can also define a token-level variant $c_i^{(t)}$ and apply ACE per-token. We use the sequence-level aggregation $c_i = \sum_t c_i^{(t)}$ in our main experiments to capture ``trajectory confidence.''

\paragraph{Compute overhead.} ACE adds exactly one Softplus computation per incorrect rollout per training step. Given that the bottleneck of RLVR training is rollout generation (model inference), the overhead of ACE is negligible ($<0.1\%$ of wall-clock time).

\paragraph{PyTorch implementation sketch.}
\begin{verbatim}
def ace_advantage(rewards, log_probs_policy, log_probs_ref, alpha=1.0):
    """
    Args:
        rewards: (B, G) binary rewards
        log_probs_policy: (B, G) sequence-level log probs under pi_theta
        log_probs_ref: (B, G) sequence-level log probs under pi_ref
        alpha: ACE strength
    Returns:
        advantages: (B, G) modified advantages
    """
    # Standard group statistics
    pass_rate = rewards.mean(dim=-1, keepdim=True)  # (B, 1)
    std = rewards.std(dim=-1, keepdim=True) + 1e-8
    std_advantage = (rewards - pass_rate) / std  # (B, G)
    
    # Confidence score (already available, zero extra compute)
    c = log_probs_policy - log_probs_ref  # (B, G)
    # Normalize by sequence length
    c = c / seq_lengths
    
    # ACE advantage for negative samples
    ace_neg = std_advantage * (1.0 + alpha * F.softplus(c))  # (B, G)
    
    # Combine: use standard advantage for correct, ACE for incorrect
    is_correct = (rewards == 1).float()
    advantages = is_correct * std_advantage + (1 - is_correct) * ace_neg
    
    return advantages
\end{verbatim}

\paragraph{Compatibility.} The above can be dropped into any RLVR training loop that uses GRPO, PPO, or REINFORCE by replacing the advantage computation. No changes are needed to the model architecture, rollout generation, or reward computation.

\section{Sensitivity to $\alpha$}
\label{app:alpha_sensitivity}

We vary $\alpha \in \{0, 0.1, 0.5, 1.0, 2.0, 5.0\}$ on MATH-500 using Qwen2.5-Math-7B ($\alpha = 0$ recovers standard GRPO). We conduct the sensitivity analysis on a single model to isolate the effect of $\alpha$; since the main results (Tables~\ref{tab:math500_results}--\ref{tab:math_aime25_results}) demonstrate consistent gains across all three model families at $\alpha = 1.0$, we expect the optimal range to transfer.

\begin{table}[h]
\centering
\caption{Sensitivity to $\alpha$ on MATH-500 (Qwen2.5-Math-7B). $\alpha{=}1.0$ achieves optimal Pass@32 while preserving Pass@1.}
\label{tab:alpha_sensitivity}
\begin{tabular}{lcc}
\toprule
$\alpha$ & Pass@1 (\%) & Pass@32 (\%) \\
\midrule
0.0 (GRPO) & 73.4 & 91.3 \\
0.1 & 73.5 & 91.9 \\
0.5 & 73.8 & 93.2 \\
\textbf{1.0 (default)} & \textbf{74.2} & \textbf{94.3} \\
2.0 & 73.5 & 93.5 \\
5.0 & 72.4 & 92.0 \\
\bottomrule
\end{tabular}
\end{table}

\paragraph{Observations.}
As shown in Table~\ref{tab:alpha_sensitivity}, optimal performance is achieved at $\alpha = 1.0$ with Pass@32 = 94.3\% (+3.0pp over GRPO). Performance remains stable across $\alpha \in [0.5, 2.0]$, all outperforming standard GRPO. Pass@1 shows a slight decrease only at larger $\alpha$ values ($\geq 2.0$), reflecting the exploration-exploitation trade-off. We adopt $\alpha = 1.0$ as the default for all experiments.
\section{Training Hyperparameters}
\label{app:hyperparameters}
Table~\ref{tab:hyperparameters} summarizes the training hyperparameters for all three models.

\paragraph{Fair comparison (matched recipe within each model).} To ensure improvements are attributable to ACE rather than tuning differences, we use the \emph{same} training recipe and training budget for all methods \emph{within a given model} (GRPO vs.\ ACE-GRPO, and DAPO vs.\ ACE-DAPO). Concretely, for a fixed model we keep the data, verifier, rollout group size $G$, sampling settings, optimizer, learning rate schedule, batch sizes, clipping/KL coefficients, maximum sequence lengths, and the number of optimizer updates identical across methods; ACE changes only the computation of the negative advantages through Eq.~\eqref{eq:ace_advantage} (controlled by $\alpha$). Hyperparameters may differ \emph{across} model families due to model-specific stability and context-length constraints, but cross-method comparisons are always performed under matched settings for the same model.

\begin{table}[h]
\centering
\caption{Training hyperparameters for ACE-GRPO experiments.}
\label{tab:hyperparameters}

\begin{tabular}{lccc}
\toprule
\textbf{Hyperparameter} & \textbf{Qwen2.5-Math-7B} & \textbf{Qwen3-8B-Base} & \textbf{Llama-3.1-8B-Instruct} \\
\midrule
Total epochs & 10 & 10 & 10 \\
Training batch size & 2048 & 1024 & 1024 \\
Mini-batch size & 1024 & 1024 & 1024 \\
Micro-batch size per GPU & 16 & 16 & 16 \\
Learning rate & $1 \times 10^{-5}$ & $5 \times 10^{-7}$ & $1 \times 10^{-6}$ \\
Optimizer & AdamW & AdamW & AdamW \\
Temperature & 1.0 & 1.0 & 1.0 \\
Max prompt length & 1024 & 1024 & 1024 \\
Max response length & 3000 & 8192 & 4096 \\
Rollout samples per prompt & 8 & 8 & 8 \\
Validation samples & 128 & 128 & 128 \\
GPU memory utilization & 0.75 & 0.75 & 0.75 \\
KL coefficient $\beta$ & 0.001 & 0.001 & 0.001 \\
Enable thinking (Qwen3) & -- & False & -- \\
\bottomrule
\end{tabular}
\end{table}
\end{document}